 \documentclass[sn-nature,iicol]{sn-jnl}

\usepackage{graphicx}%
\usepackage{multirow}%
\usepackage{amsmath,amssymb,amsfonts}%
\usepackage{amsthm}%
\usepackage{mathrsfs}%
\usepackage[title]{appendix}%
\usepackage{xcolor}%
\usepackage{textcomp}%
\usepackage{manyfoot}%
\usepackage{booktabs}%
\usepackage{algorithm}%
\usepackage{algorithmicx}%
\usepackage{algpseudocode}%
\usepackage{listings}%
\usepackage{longtable}
\usepackage{tikz}
\usepackage{subfiles}
\usepackage{subfigure}
\usepackage{paralist}
\usepackage{textcomp}
\usepackage{soul}

\newtheorem{Thm}{Theorem}[subsection]

\newtheorem{Lem}[Thm]{Lemma}

\newtheorem{Cor}[Thm]{Corollary}
\theoremstyle{definition} 

\newtheorem{Def}[Thm]{Definition}

\newtheorem{Rem}[Thm]{Remark}

\begin{document}

\title[Article Title]{On the approximation capability of GNNs in node classification/regression tasks}


\author*[1]{\fnm{Giuseppe Alessio} \sur{D'Inverno}}\email{dinverno@diism.unisi.it}

\author[1]{\fnm{Monica} \sur{Bianchini}}\email{monica.bianchini@unisi.it}

\author[1]{\fnm{Maria Lucia} \sur{Sampoli}}\email{marialucia.sampoli@unisi.it}

\author[1]{\fnm{Franco} \sur{Scarselli}}\email{franco.scarselli@unisi.it}

\affil*[1]{\orgdiv{Department of Information Engineering and Mathematics}, \orgname{University of Siena}, \orgaddress{\street{Via Roma 56}, \city{Siena}, \postcode{53100}, \state{(SI)}, \country{Italy}}}




\abstract{Graph Neural Networks (GNNs) are a broad class of connectionist models for graph processing. Recent studies have shown that GNNs can approximate any function on graphs, modulo the equivalence relation on graphs defined by the Weisfeiler--Lehman (WL) test. However, these results suffer from some limitations, both because they were derived using the Stone--Weierstrass theorem -- which is existential in nature --, and because they assume that the target function to be approximated must be continuous.  Furthermore, all current results are dedicated to graph classification/regression tasks, where the GNN must produce a single output for the whole graph, while also node classification/regression problems,  in which an output is returned for each node, are very common.
In this paper, we propose an alternative way to demonstrate the approximation capability of GNNs that overcomes these limitations. Indeed, we show that GNNs are universal approximators in probability for node classification/regression tasks, as they can approximate
any  measurable function that satisfies the  1--WL equivalence on nodes.
The proposed theoretical framework allows the approximation of generic discontinuous target functions and also suggests the GNN
architecture that can reach a desired approximation. In addition, we provide a bound on the number of the GNN layers required to achieve the desired degree of approximation, namely $2r-1$, where $r$ is the maximum number of nodes for the graphs in the domain.}




\maketitle

\section{Introduction}
Graph processing is becoming pervasive in many application domains, such as social networks, Web applications, biology and finance. Intuitively, graphs allow to represent patterns along with their relationships.
Indeed, graphs can naturally encode high--valued information that is hard to represent with vectors or sequences, the most common data structures used in Machine Learning (ML). 
Graph Neural Networks (GNNs) are a class of machine learning models that can process information represented in the form of graphs. In recent years, the interest in GNNs has grown rapidly and numerous new models and applications have emerged~\cite{wu2020}.  The first GNN model was introduced  in~\cite{GNN}. Later, several other approaches have been proposed, including Spectral Networks~\cite{bruna2013}, Gated Graph Sequence Neural Networks~\cite{li2015gated},
Graph Convolutional Neural Networks~\cite{kipf2016},
GraphSAGE~\cite{hamilton2017inductive}, Graph attention networks~\cite{GAT}, and Graph Networks~\cite{Battaglia2018}.
However, despite the differences among the various GNN models, 
most adopt the same computational scheme, based on a local aggregation mechanism.
The information related to a node is stored into a feature vector, which is updated recursively by aggregating the feature vectors of neighboring nodes. 
After $k$ iterations,
the feature vector of a given node $v$ captures both structural information and attributes
of the nodes in the $k$--hop neighborhood of $v$. At the end of the learning process, the node feature vectors can be used to classify or to cluster the objects/concepts represented by a (some) node(s), or by the whole graph. 

Recently, a great effort has been devoted to  study   the 
expressive power of GNNs~\cite{sato2020s}. Such a  theoretical property has an important 
impact in machine learning, since it defines what are the applications that can be faced by a neural 
network model, it can explain observed limitations in experiments and, finally, it can suggest novel advancements to improve the considered model. 
In GNNs, the capabilities and the limitations of the model primarily depend on the local computational framework, since GNNs can take into account both the connectivity and the features of the neighboring nodes,
but they may not be able to distinguish between nodes having similar neighborhoods. Therefore, a fundamental question is to define which graphs (nodes) can be distinguished by a GNN, i.e. for which input graphs (nodes) the GNN produces different encodings. In~\cite{xu2018powerful}, GNNs are proved to be as powerful as the Weisfeiler--Lehman graph isomorphism test (1--WL)~\cite{leman1968}. 
Such an algorithm allows to test whether two graphs are isomorphic or not~\footnote{It is worth noting that the 1--WL test is inconclusive, since there exist pairs of graphs that the test recognizes as isomorphic even if they are not.}. The 1--WL  algorithm is based on a graph signature which is obtained by assigning a color to each node, where the graph coloring is achieved by iterating a local aggregation function.
More generally, there exists a hierarchy of algorithms, called  1--WL, 2--WL, 3--WL, etc.,
 which recognizes larger and larger classes of graphs.
It has been shown that a GNN can simulate the 1--WL test, provided that a sufficiently general aggregation function is used, but the basic GNN model cannot implement  higher order tests~\cite{morris2019}. Consequently, the  1--WL test characterizes both the expressiveness and limitations of GNNs, defining the classes of graphs/nodes that GNNs can distinguish.

Another important aspect is the study of the approximation capability of GNNs.
Formally, in
node classification/regression tasks, a GNN implements a function $\varphi(\mathbf{G},v)\rightarrow \mathbb{R}^m$
that takes in input a graph $\mathbf{G}$ and returns an output at each node. Similarly, in graph classification/regression tasks, a GNN implements a function $\varphi(\mathbf{G})\rightarrow \mathbb{R}^m$.
In both cases, the objective is to define which classes of functions can be approximated by a GNN.

In~\cite{Comp_GNN}, the  approximation capability of the original GNN  model (OGNN), namely the first GNN model to be proposed, has been studied using the concept of unfolding trees and unfolding equivalence. The unfolding tree $\mathbf{T}_v$,
with root node $v$, is constructed by unrolling the graph starting from $v$ (see Fig. \ref{fig:unfolding1}).
Intuitively, $\mathbf{T}_v$ exactly describes the information used by the GNN at node $v$
and can be employed to study the expressive power of GNNs in node classification/regression tasks.
The unfolding equivalence is, in turn,  an equivalence relationship defined between nodes having the same unfolding tree.  In~\cite{Comp_GNN}, it was proved that OGNNs can approximate in probability, up to any degree of precision, any measurable function $\tau(\mathbf{G},v)\rightarrow \mathbb{R}^m$  that respects the unfolding equivalence, namely, that produces the same outputs on equivalent nodes.
Currently, unfolding trees --- also termed {\em computation graphs}~\cite{garg2020generalization} --- are widely used to study the GNN expressiveness.
Universal approximation results have been proved  for  Linear Graph Neural Networks~\cite{azizian2020expressive,maron2018invariant}, Folklore Graph Neural Networks~\cite{maron2019provably} and, more generally, for  a large class of  GNNs~\cite{xu2018powerful,azizian2020expressive}
that includes most of the recent architectures, also considered in this paper.

\begin{figure*}[ht]
\centering
 \includegraphics[width=.8 \linewidth]{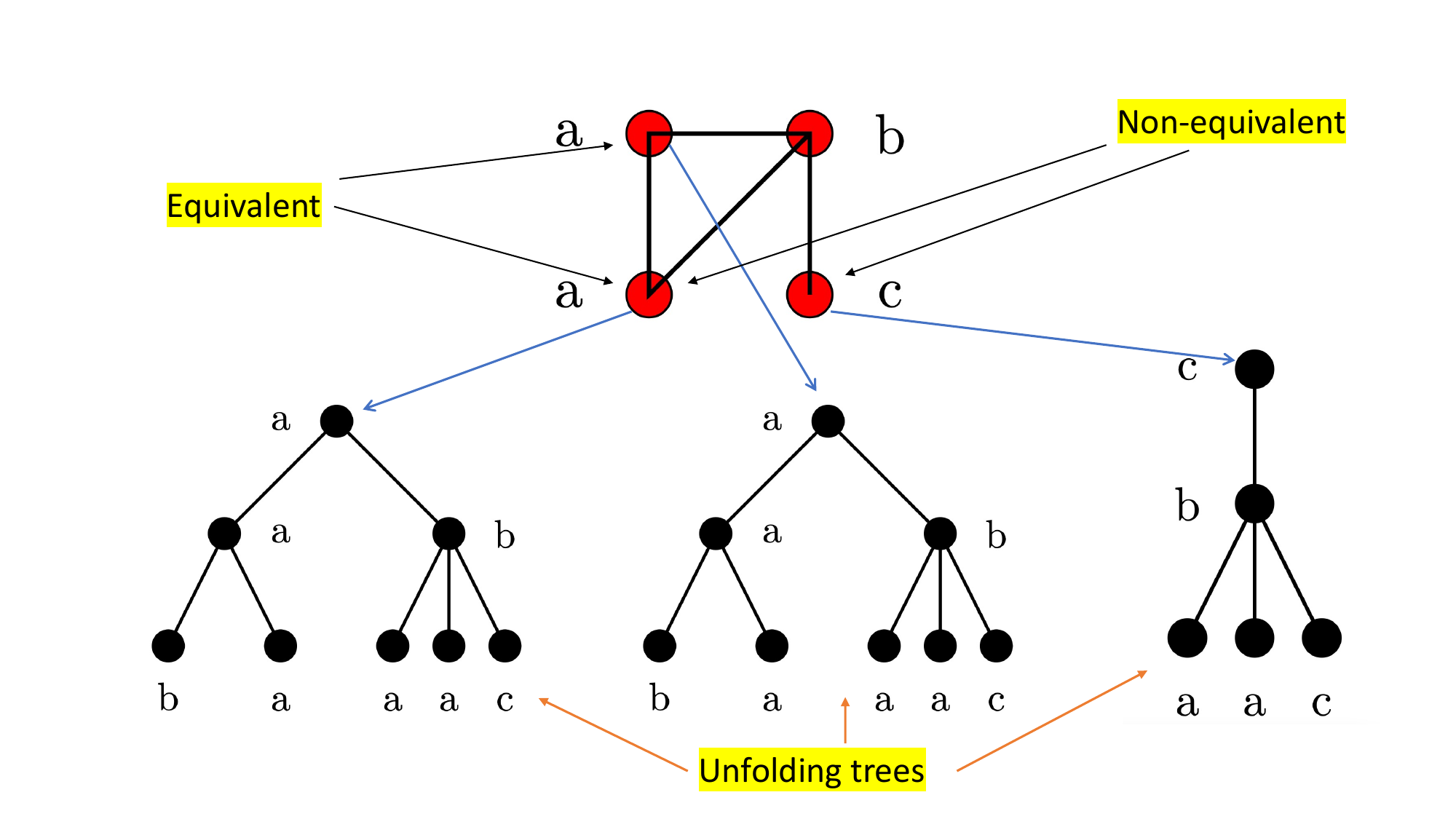}
\caption{An example of a graph with some unfolding trees. The symbols outside the nodes represent features. The two nodes on the left part of the graph are equivalent and have equivalent unfolding trees. }
\label{fig:unfolding1}
\end{figure*}

Despite many advances in research on
approximation theory for GNNs, there are still open problems to be investigated. 
First of all, the most general results available on modern GNNs are based on
the Stone--Weierstrass theorem and state that the functions which can be approximated by GNNs are dense in the invariant continuous function space,
modulo the 1--WL test \cite{azizian2020expressive}. However, the Stone--Weierstrass theorem is existential in nature, so that, given a target function to be approximated, it does not allow to
construct a GNN architecture that can reach the desired approximation --- defining, for example, the number of its layers, and the feature dimension required to build the approximator.
Moreover,   the current results apply only to continuous functions on node/edge labels,
which are defined on a compact subset of $\mathbb{R}^{{\mathbf{L}}}$, a fact that may not hold in practical application domains, since, for instance,  the function to be approximated may show step--wise behavior with respect to some  inputs. Finally, all the results on the expressive capacity of modern GNN models
are dedicated to graph classification/regression tasks, but node classification/regression problems are also widely present in practical applications and it is important to generalize the theoretical results on expressivity to them as well. In addition,  it is useful to study
the relationships between unfolding trees and the 1--WL test in this context. Indeed, it can be observed that the Weisfeiler--Lehman test assigns a color to all the nodes of a graph to make them distinguishable, and it can be naturally expected that the equivalence classes defined by the colors are related to those defined by the unfolding trees. In fact,  it has been proved that the
two mechanisms, colors or unfolding trees, produce the same profiles for graphs~\cite{krebs2015universal}, namely the same number of nodes per equivalence class, but whether they produces exactly the same profiles with respect to single nodes, i.e., nodes get assigned the same equivalence class, is still an open problem. A formal and precise answer to this question will allow us to use the two frameworks in a targeted or exchangeable way in the context of node classification/regression tasks.

In this work, we present an alternative approach to study the approximation capability of 
recent GNNs that allows to answer to the above questions.

The main contributions of this paper are listed below.

\begin{itemize} 
\item
We prove that, on connected graphs, modern  GNNs, realizing node--focused functions, are  capable of approximating, in probability and up to any precision, any measurable function that respects the 1--WL equivalence. 
Intuitively, this means that GNNs are a kind of universal approximators for functions on the nodes of the graph, modulo the limits enforced by the 1--WL test.
Such a result describes the GNN capability for node classification/regression tasks.
\item 
The presented proof is the most general  on the GNN approximation capability that we are aware of, since it
holds for generic graphs with real feature vectors and for a broad class of GNNs, which includes most of the  current models.
Moreover, it is assumed that the target function is measurable, which
permits the approximation of discontinuous and more complex functions w.r.t. existing results, e.g.
\cite{jegelka2022theory}.  Finally, the proof is based on a technique that allows us to deduce information on the architecture of the GNN  that can reach the desired approximation. Such an information cannot be derived with the Stone--Weierstrass theorem and  includes, for instance, hints on the number of iterations, the number of layers,  the dimension of hidden features,  and the type of the network to be used to implement the aggregation function.
\item  It is shown that, in order to reach any desired approximation accuracy, 
a single real hidden feature is sufficient, the aggregation network must contain at least 
one hidden layer, and the GNN must adopt at least $2r-1$ iterations, namely the GNN must include
$2r-1$ layers, where $r$ is the maximum number of nodes of any graph in the domain. The latter bound
on GNN iterations/layers can be surprising    because we may expect that $r$ iterations are sufficient to diffuse the
information on the whole graph.  We will  clarify that such  a  bound is due to the nature of node classification/regression tasks. Actually, $r$ iterations are sufficient for graph classification/regression tasks, but they are not enough for node--focused tasks, which are more expensive from a computational point of view.
\item A set of experiments has been carried out in order to show that GNNs, if their architectures are sufficiently general, can approximate any function, modulo the unfolding equivalence/1--WL test, up to a desired degree of precision, so as suggested by the proposed theoretical results. 
\end{itemize}
We remark that understanding the approximation power of GNNs is fundamental in order to 
explain GNN  limitations and capabilities in practical applications and to have suggestions for designing novel advanced models.
The present contribution aims to fill the gap left in literature on the characterization of the expressive power of modern GNNs under some crucial aspects, such as the universality  on real-attributed graphs,   the approximation capabilities on node-focused tasks,  the number of required layers. This finally contributes to
a more thorough comprehension of the GNN machine learning framework.

The rest of the paper is organized as follows. In Section~\ref{rel_work}, some related work is described. Notation and basic concepts are introduced in Section~\ref{notation}, while  Section~\ref{main_results} presents the main contribution of this paper. In Section \ref{section_experiments}, we present the experiments  conducted to validate our theoretical results.
Finally, Section~\ref{ConcFut} gives some conclusive remarks and presents future perspectives. To make the reading more fluid, the proofs are collected in the Appendix.

\section{Related Work}\label{rel_work}
Great attention has recently been paid to the Weisfeiler--Lehman test and its correlation with the expressiveness of GNNs. Xu \textit{et al.} \cite{xu2018powerful} have shown that message passing GNNs are at most as powerful as the 1--WL test; this upper bound could be overcome by injecting the node identity in the message passing procedure, as implemented in \cite{you2021identity}.  
Morris \textit{et al.} \cite{morris2019} have gone beyond the 1--WL test, implementing $k$--order WL tests as message passing mechanisms into GNNs.
In \cite{sato2020s}, the WL test mechanism applied to GNNs is studied within the paradigm of unfolding trees (also called \textit{computational graphs}), without really establishing an equivalence between the two concepts, so as in~\cite{zhang2021nested} (where the unfolding trees are called \textit{rooted subgraphs}). In \cite{alon2020bottleneck}, it is shown that the Weisfeiler--Lehman test tends to oversquash the information coming from the neighbours; moreover, it is claimed that GNNs with at least $K$ layers, where $K$ is the diameter of the graphs in the dataset, do not suffer from under--reaching, which means that the information cannot travel farther than $K$ edges along the graph. Nevertheless, a theoretical proof that GNNs succeed in overcoming the under--reaching behavior is not provided.

Universal approximation properties have been demonstrated for several GNN settings.
The OGNN \cite{GNN} model was proved to be a universal approximator on graphs preserving the unfolding equivalence in \cite{Comp_GNN}. Universal 
 approximation is shown for GNNs with random node initialization in \cite{abboud2020} while, in~\cite{xu2018powerful}, they are proved to be able to encode any graph with countable input  features.
The universal approximation property has been extended to Folklore Graph Neural Networks in~\cite{maron2019provably}, to Linear Graph Neural Networks
and general GNNs in \cite{azizian2020expressive,maron2018invariant}, both in the invariant and equivariant case, but without any reference to the required number of layers.
A relation between the graph diameter and the
computational power of GNNs has been established  in~\cite{loukas2019graph}, where the GNNs are assimilated to the so--called LOCAL models~\cite{Angluin80,linial92,Naor93} and it is proved that a GNN with a number of layers larger than the diameter of the graph can compute any Turing function of the graph. Nevertheless,  no information on the aggregation function characterization is given. 
The generalization capability of GNNs has been also studied using different approaches, which include the Vapnik--Chervonenkis dimension for OGNNs \cite{vapnik2018}, and the uniform stability \cite{zhou2021} and  Rademacher complexity \cite{garg2020generalization} for modern GNNs. Designing GNN architectures that provide good generalization along with good expressive power is a hot research topic (see, e.g., \cite{omri2020}). Moreover, an extensive survey on the theory of Graph Neural Networks can be found in ~\cite{jegelka2022theory}.

The results presented in this work differ from what can be found in literature mainly because we prove the GNN ability to approximate measurable functions based on a proof which is constructive, i.e. capable of suggesting the network architecture that will guarantee a given approximation.

\section{Preliminaries} \label{notation}
In this section, we introduce the required notation and the basic definitions used throughout the manuscript.

\subsection{Graphs}
A graph $\mathbf{G}$ is a pair $(\mathbf{V},\mathbf{E})$, where $\mathbf{V}$ is the set of \textit{vertices} or \textit{nodes} and $\mathbf{E}$ is the set of \textit{edges} between nodes in $\mathbf{V}$. Graphs are \textit{directed} or \textit{undirected}, according to whether the edge $(v,u)$ is different from the edge $(u,v)$ or not. Moreover, a graph is \textit{connected} if there is a path from any node to any other node in the graph. In the following, we assume that graphs are undirected and connected.

The set ${ne}[v]$ is the \textit{neighborhood} of $v$, i.e. the set of nodes connected to $v$ by an edge, while ${ne_i}(v)$ denote the $i$--th neighbor of $v$ --- the set of all nodes connected to $v$ with a path of length $i$. Finally, $|\mathbf{G}|$ defines the cardinality of the set of vertices in $\mathbf{G}$.
From now on, we will always consider graphs with finite cardinality, i.e., $|\mathbf{G}|=\mathbf{L} < \infty$.

 Nodes may have attached features, collected into vectors called \textit{labels}, identified with $\mathbf{\ell}_v \in \mathbb{R}^{{\mathbf{L}}}$.

\subsection{Graph neural networks}

Graph Neural Networks adopt  a local computational mechanism to process graphs. The information related to a node $v$ is stored into a feature vector $\mathbf{h}_{v}\in\mathbb{R}^m$,
which is updated recursively by combining the feature vectors of neighboring nodes. 
After $k$ iterations, the feature vector  $\mathbf{h}_v^k$ is supposed to contain a representation of both the structural information and the node information within a $k$--hop neighborhood. After processing is complete, the node feature vectors can be used to classify the nodes or the entire graph.

More rigorously, in this paper,  we consider GNNs that use the following general updating scheme:
\begin{align} \label{Def_GNN}
\mathbf{h}^k_v = & \text{\small{COMBINE}}^{(k)}\big(\mathbf{h}^{k-1}_{v}, &\\
& \text{\small{AGGREGATE}}^{(k)} \{\!\{\mathbf{h}^{k-1}_{u}, \, u \in ne[v]\}\!\}\big) & \nonumber
\end{align}
\noindent
where the node feature vectors are initialized with the node  labels, i.e.,  $\mathbf{h}^0_v= \mathbf{\ell}_v\in\mathbb{R}^{{\mathbf{L}}}$ for each $v$. Here, differently from other approaches, we assume that 
labels can contain real numbers.
Moreover, $\text{AGGREGATE}^{(k)}$ is a function which aggregates the node features obtained in the ($k-1$)--th iteration, and $\text{COMBINE}^{(k)}$ is a function that combines the aggregation of the neighborhood of a node with its feature at the ($k-1$)--th iteration.
In graph classification\slash regression tasks, the GNN is provided with a final READOUT layer that produces the output combining all the feature vectors at the last iteration $K$:
\begin{equation}
\mathbf{o} \; = \; \text{READOUT}( \{\mathbf{h}^{K}_v, \; v \in \mathbf{V}\})\,
\label{eq:heq}
\end{equation}
whereas, in node classification\slash regression tasks, the READOUT layer produces an output for each node, based on its features:
\begin{equation}
\mathbf{o}_v \; = \; \text{READOUT}( \mathbf{h}^{K}_v)\,
\label{eq:h2eq}
\end{equation}

In this paper, we will focus mainly on node classification/regression tasks.
The learning domain of the GNN will be denoted  by the  graph--node pair $\mathcal{D}= \mathcal{G} \times \mathcal{V}$, where $\mathcal{G}$ is a set of graphs 
and $\mathcal{V}$ is a subset of their nodes. Therefore, the function $\varphi$, implemented by the GNN, takes in input a graph $\mathbf{G}$ and  one of its nodes $v$, and  returns an output $\varphi(\mathbf{G}, v)\in  \mathbb{R}^o$, where $o$ is the output dimension. 

The framework described by Eqs.~(\ref{Def_GNN})--(\ref{eq:h2eq}) is commonly used to study theoretical properties of modern GNNs (see e.g.~ \cite{xu2018powerful}).
The class of models covered by such a framework is rather wide and includes, for example,  GraphSAGE \cite{hamilton2017inductive}, GCN \cite{kipf2016},  GATs \cite{GAT}, GIN \cite{xu2018powerful}, ID-GNN \cite{you2021identity}, and GSN \cite{bouritsas2020improving}. 

It is worth mentioning that the OGNN model is not formally covered, both because in OGNNs
the input of $\text{AGGREGATE}^{(k)}$ and $\text{COMBINE}^{(k)}$ contains the node labels $\mathbf{\ell}_v$  and possibly also the  edge features, and because the node features are not initialized to $\mathbf{\ell}_v$. Other models, such as MPNN~\cite{gilmer2017},
NN4G~\cite{micheli2009} and GN~\cite{Battaglia2018} are not included as well for similar reasons. 
Of course, Eq.~(\ref{eq:heq}) could easily be extended
to include also OGNNs and the models mentioned above, but here we prefer not to complicate the proposed framework to keep the notation and proofs simple.

\subsection{Unfolding trees and unfolding equivalence}

 \textit{Unfolding trees}~\footnote{Unfolding trees are also referred to as \textit{computational graphs}~\cite{garg2020generalization} or \textit{search trees} ~\cite{sato2020s,xu2018powerful}.}  and \textit{unfolding equivalence} are two concepts that have been introduced in~\cite{Comp_GNN}  with the aim of capturing the expressive power of the OGNN model.
Intuitively, 
an \textit{unfolding tree} $\mathbf{T}_v^d$ is the tree obtained by unfolding the graph up to the depth $d$, using the node $v$ as its root.
Fig. \ref{fig:unfolding1} shows some examples of unfolding trees. In the following, a  formal recursive definition is provided.
\begin{Def}
The unfolding tree $\mathbf{T}_v^d$ of a node $v$ up to depth $d$ is
\begin{equation*}
\mathbf{T}_v^d \; = \; \left \{ \begin{array}{lcr} 
\text{Tree}(\mathbf{\ell}_v) & \text{if} & d  \; = \; 0 \\
\text{Tree}(\mathbf{\ell}_v,\mathbf{T}_{ne[v]}^{d-1}) & \text{if} & d \; > \; 0
\end{array} \right.
\end{equation*}
\noindent
where 
$\text{Tree}(\mathbf{\ell}_v)$ is a tree constituted of a single node with label $\mathbf{\ell}_v$ and $\text{Tree}(\mathbf{\ell}_v,\mathbf{T}_{ne[v]}^{d-1})$ is the tree with the root node labeled with $\mathbf{\ell}_v$ and having sub--trees $\mathbf{T}_{ne[v]}^{d-1}$. The set $\mathbf{T}_{ne[v]}^{d-1}=\{\mathbf{T}_{u_1}^{d-1},\mathbf{T}_{u_2}^{d-1}, \dots\}$ collects all unfolding trees having depth $d-1$, with $u_i \in ne[v], \, \forall i$.\\
Moreover, the \textit{unfolding tree of $v$},  $\mathbf{T}_v= \lim \limits_{d \rightarrow \infty} \mathbf{T}_v^d$, is obtained  by merging all unfolding trees $\mathbf{T}_v^d$ for any $d$.
\hfill $\blacksquare$
\end{Def}

\vspace{0.3cm}
Note that, since a GNN adopts a local computation framework, its knowledge about the graph is updated step by step, every time Eq.~(\ref{Def_GNN}) is applied. Actually,  at the first step, $k=0$, the feature vectors $\mathbf{h}_v^0$ depends only on the local label. Then, at step $k$, 
the GNN updates the feature vector  $\mathbf{h}_v^k$ using the neighbour data, with
the node feature vector that depends on the $k$--distant neighbourhood of $v$. Thus, intuitively, the unfolding tree $ \mathbf{T}_v^k$
 describes  the  information that is theoretically available to the GNN at node $v$ and step $k$.
 Such an observation has been used in \cite{Comp_GNN} to study the expressive power of the OGNN model and will be used also in this  paper for the same purpose.

In this context, two questions have been studied. 
\begin{enumerate}
    \item[(1)] Can GNNs compute and store into the node features a coding of the unfolding trees, namely can GNNs store all the theoretically available information? 
    \item[(2)] Since unfolding trees are  different from the input graphs, how does this affect the GNN expressive power?
    \end{enumerate}
Regarding the first question, it has been shown that indeed both OGNNs and modern GNNs can compute and store in the node features a coding of the unfolding trees, provided that the appropriate network architectures are used in $\text{COMBINE}^{(k)}$
and $\text{AGGREGATE}^{(k)}$~\cite{sato2020s,Comp_GNN,xu2018powerful}. Regarding question (2), we can easily argue that if two nodes 
have the same unfolding tree, then GNNs  produce the same encoding on those nodes. Such a fact highlights an evident limitation of the expressive power of GNNs. The unfolding equivalence is a
formal tool designed to capture such a limit: it is an equivalence relation that brings together nodes with the same unfolding tree, namely it groups nodes that cannot be distinguished by GNNs.

\begin{Def}\label{def:unf_eq}
Two nodes $u$, $v$ are said to be \textit{unfolding equivalent} $u \backsim_{ue} v$, if $\mathbf{T}_u = \mathbf{T}_v$. Analogously,
two graphs $\mathbf{G}_1, \mathbf{G}_2$ are said to be \textit{unfolding equivalent} $\mathbf{G}_1 \backsim_{ue} \mathbf{G}_2$, if there exists a bijection between the nodes of the graphs that respects the partition induced by the unfolding equivalence on the nodes~\footnote{For the sake of simplicity, and with notation overloading, we adopt the same symbol $\backsim_{ue}$ both for the equivalence between graphs and the equivalence between nodes.}.
\hfill $\blacksquare$
\end{Def}

\vspace{.3cm}
Since GNNs have to fulfill the unfolding equivalence, also the functions on graphs that they can realize share this limit. In our results on the approximation capability of GNNs, our focus is on functions that preserve
the unfolding equivalence.
Those functions are general enough except that they  produce the same output on equivalent nodes.

\subsection{The color refinement algorithm and the Weisfeiler--Lehman test}

The  \textit{first order Weisfeiler--Lehman test}  (\textit{1--WL test} in short) \cite{leman1968} is
a method  to test whether two graphs are isomorphic, based on a graph coloring algorithm, called \textit{color refinement}.  
The coloring algorithm is applied in parallel on the two graphs. Each node
keeps a state (or color) that gets refined in each iteration by
aggregating information from its neighbors' state. The refinement stabilizes after a few iterations and it outputs a
representation of the graph. Two graphs with different representations, i.e. with a different number of nodes for each color, are not isomorphic.
Conversely, if the numbers match, then the graphs are \textit{possibly} isomorphic. Note that the test is not conclusive in the case of a positive answer, as the graphs may still be non--isomorphic. Actually, the algorithm just provides an approximate solution to the problem of graph isomorphism.
 
There exist  different versions of the coloring algorithm: in this paper, we adopt a  coloring scheme in which also the node labels are considered. Since
GNNs process both the structure and the labels of the graphs, it is useful to consider both these sources of  information, in order  to analyse the GNN expressive power. Such an approach has been used, for example, in~\cite{sato2020s}. More precisely, the coloring is carried out by an iterative algorithm which,  at each iteration, computes a \textit{node coloring} $c_l^{(t)} \in \Sigma$, being $\Sigma$ a subset of values representing the colors. The node colors are initialized on the basis of the node features and then they are updated using the coloring from the previous iteration. The algorithm is sketched in the following.
\begin{enumerate}
\item
At iteration 0, we set 
$$c_v^{(0)}=\text{HASH}_0(\mathbf{\ell}_v)$$
where $\text{HASH}_0: \mathbb{R}^q \rightarrow \Sigma$ is a function that bijectively encodes real features using colors. In case of unattributed graphs, we assume $q=1$ and $\ell_v = 1 \;,  \forall v \in V$, $\forall G=(V,E) \in \mathcal{G}$.
\item
For any iteration $t >0$, we set
\begin{equation*}
c_v ^{(t)}=\text{HASH}(c_v^{(t-1)},   \{\!\{ c_n^{(t-1)}| n \in ne[v]  \}\!\}  )
\end{equation*}
where $\text{HASH}:  \Sigma \times \Sigma^* \rightarrow \Sigma$ is a function that  bijectively maps the input pairs to a unique value in $\Sigma$.  The notation $ \{\!\{ \cdot \}\!\} $ represents \textit{multisets}, which can be formulated, in our setting, without loss of generality, as ordered sequences of elements in $\Sigma$, i.e. they belong to  $\Sigma^* = \bigcup\limits_{n\geq1} \Sigma^n$.  Moreover, we assume that the same $\text{HASH}$ function is used for all the iterations\footnote{In \cite{kiefer2020power}, it is assumed that the $\text{HASH}$ functions are different at each step, so that  the algorithm can reuse the same finite set of colors, e.g., denoted by the integer numbers
$1$ to $r$, where $r$ is the number of nodes in the graph. This can be achieved by bijectively
re--mapping the colors after each refinement step. The two algorithms are equivalent w.r.t. the goal
of isomorphism testing. Here, we prefer to assume that a unique $\text{HASH}$ function is adopted because
such an assumption will simplify our discussion about the properties of the algorithm.} .
\end{enumerate}

In order to compare two graphs $\mathbf{G}'=(\mathbf{V}',\mathbf{E}')$,
$\mathbf{G}''=(\mathbf{V}'',\mathbf{E}'')$, the  coloring refinement is applied in parallel on
 $\mathbf{G}' ,\mathbf{G}'$, and, at each step, the color profiles generated on each graph are compared, namely, $\{ c_n^{(t)}| n \in \mathbf{V}'\} =\{(c_m^{(t)}| m \in \mathbf{V}''\} $ is evaluated.
If, at any iteration, the colors of the  graphs are different, then the 1--WL test fails and
we can conclude that the  graphs are not isomorphic; otherwise, the test succeeds. The 1--WL test allows to
distinguish most non--isomorphic graphs, but may succeed on some rare examples.

In this paper, we  use the color refinement also to compare  nodes.  Thus, given two
nodes $u,v$, which in the most general case can belong to different graphs, we compare
their colors at each iteration, i.e.,  $c_u^{t}=c_v^{t}$. If, at any iteration, the node colors are different, then the 1--WL node test fails, otherwise it succeeds. Notice that the color of a node
$n$ at iteration $t$ depends on the sub--graph $\mathbf{G}_n^t$, defined by the $t$--hop neighbourhood
of $n$. Thus, intuitively, the 1--WL node test allows to check the isomorphism of the neighbourhoods
of two nodes, $\mathbf{G}_u^t\backsim\mathbf{G}_v^t$.

By the mentioned algorithms, we can easily produce a definition of 
WL--equivalence for graphs and nodes.

\begin{Def}[WL--equivalence] \label{def:wlequivalence}
Two graphs,  $\mathbf{G}'=(\mathbf{V}',\mathbf{E}')$ and
$\mathbf{G}''=(\mathbf{V}'',\mathbf{E}'')$, are said to be \textbf{WL--equivalent}, 
if they have the same multisets of colors  at each iteration of the color refinement algorithm,  i.e.,
$\{ c_n^{(t)}| n \in \mathbf{V}'\} =\{ c_m^{(t)}| m \in \mathbf{V}''\}$ for any $t$.
Analogously, two nodes, $u$ and $v$, are said to be \textbf{WL--equivalent},  $u \backsim_{WL} v$, if they have the same colors at each step of  the color refinement algorithm, i.e. $c_u^{(t)}=c_v^{(t)}$ for any $t$. 
$\hfill \blacksquare$
\end{Def}

It is interesting to observe that the color refinement procedure must be iterated until a difference in colors is detected between the compared items, either  graphs
or nodes, or until a maximum number of iterations is reached. It is well known that the color refinement of the common Weisfeiler--Lehman test, 
defined for graph comparison, can be halted when the node partition defined by colors become stable:
if the two graphs share the colors when the stability is reached, then the equality will last forever.
 More precisely, let $\pi_t(\mathbf{G})$ be the partition of the nodes of
$\mathbf{G}$ constructed by collecting in the same class the nodes that have the same color at iteration $t$.
It is not difficult to prove that the partitions become finer at each iteration, $\pi_{t-1}(\mathbf{G})\succeq\pi_t(\mathbf{G})$, and that there exists an iteration $T$ at which they become stable,
 $\pi_{T-1}(\mathbf{G})\equiv\pi_T(\mathbf{G})$, Moreover, it can be proved that $r-1$, where 
 $r$ is the number of nodes in $\mathbf{G}$, is both an upper bound and a lower bound
 for the number $T$ of iterations required to reach the stability~\cite{kiefer2020iteration}.

Note that the stability of the node partition does not imply that the colors do not change.
Actually, if the colors are not reused, as in our definition, except in the case where the graph is free of connections, new colors appear at each iteration. Intuitively, this happens because the use, at a node
 $u$,  of a new
 color, which has not been considered in the past, causes the algorithm to create new colors for the neighbors of $u$ as well: thus, new colors will be generated forever. This observation can be used to
explain why the upper bound on the iterations of the color refining procedure is different 
in the case of node or graph equivalence.
We will see that we must wait for $2r-1$ iterations before halting the procedure in the former case, 
whereas, as mentioned above, $r-1$ iterations are sufficient in the latter. 

\section{Main Results}\label{main_results}
In this section, the  main results of the paper are presented and discussed. For ease of reading the proofs of the theorems are given in the Appendix.

\subsection{Unfolding  and Weisfeiler--Lehman equivalence}
 The first proposed result regards the relationship between the unfolding  and the Weisfeiler--Lehman equivalence.
The following two theorems clarify that the two equivalence relations produce the same 
partitions of nodes and graphs. Moreover, the correspondence exists also between the  intermediate equivalences defined by, respectively,  the colors at each iteration
of the WL algorithm and  the unfolding trees having a corresponding depth. Formally, let us denote by 
$\backsim_{ue_t}$ the unfolding equivalences, at depth $t$, between nodes and graphs
that are defined as in \ref{def:unf_eq} but considering unfolding trees of depth $t$
in place of infinite trees. Similarly,  let us denote by 
$\backsim_{WL_t}$ the WL--equivalences, at iteration $t$, that are defined as in \ref{def:wlequivalence}, where only the colors  of the refinement procedure up 
to the $t$--th iteration are considered.

\begin{Thm}\label{equiv_theo1}
\hspace{0.5 cm}\\
Let ${\mathbf G}=({\mathbf V},{\mathbf E})$ be a labeled graph. Then, for each  $u,v \in {\mathbf V}$, $u \backsim_{ue} v$ holds if and only if $u \backsim_{WL} v$ holds. Moreover, for each integer $t\geq 0$, $u \backsim_{ue_t} v$ if and only if $u \backsim_{WL_t} v$. 
\hfill \qed
\end{Thm}
\begin{Thm}\label{equiv_fin}
\hspace{0.5 cm}\\
Let ${\mathbf G}_1,{\mathbf G}_2$ be two graphs.
Then, ${\mathbf G}_1 \backsim_{ue} {\mathbf G}_2$  if and only if ${\mathbf G}_1 \backsim_{WL} {\mathbf G}_2$.
Moreover, for each integer $t\geq 0$, ${\mathbf G}_1 \backsim_{ue_t} {\mathbf G}_2$ if and only if ${\mathbf G}_1 \backsim_{WL_t} {\mathbf G}_2$.
\hfill \qed
\end{Thm}

Both the unfolding equivalence and  the  WL--equivalence have been described using a recursive definition local to nodes. Figure~\ref{fig:unfolding}
shows an example in which the unfolding trees and the colors of two nodes are iteratively 
computed: in the example, the colors of the nodes become different when also the unfolding trees become different.

\begin{figure*}[h!]
 \includegraphics[width=.99\linewidth]{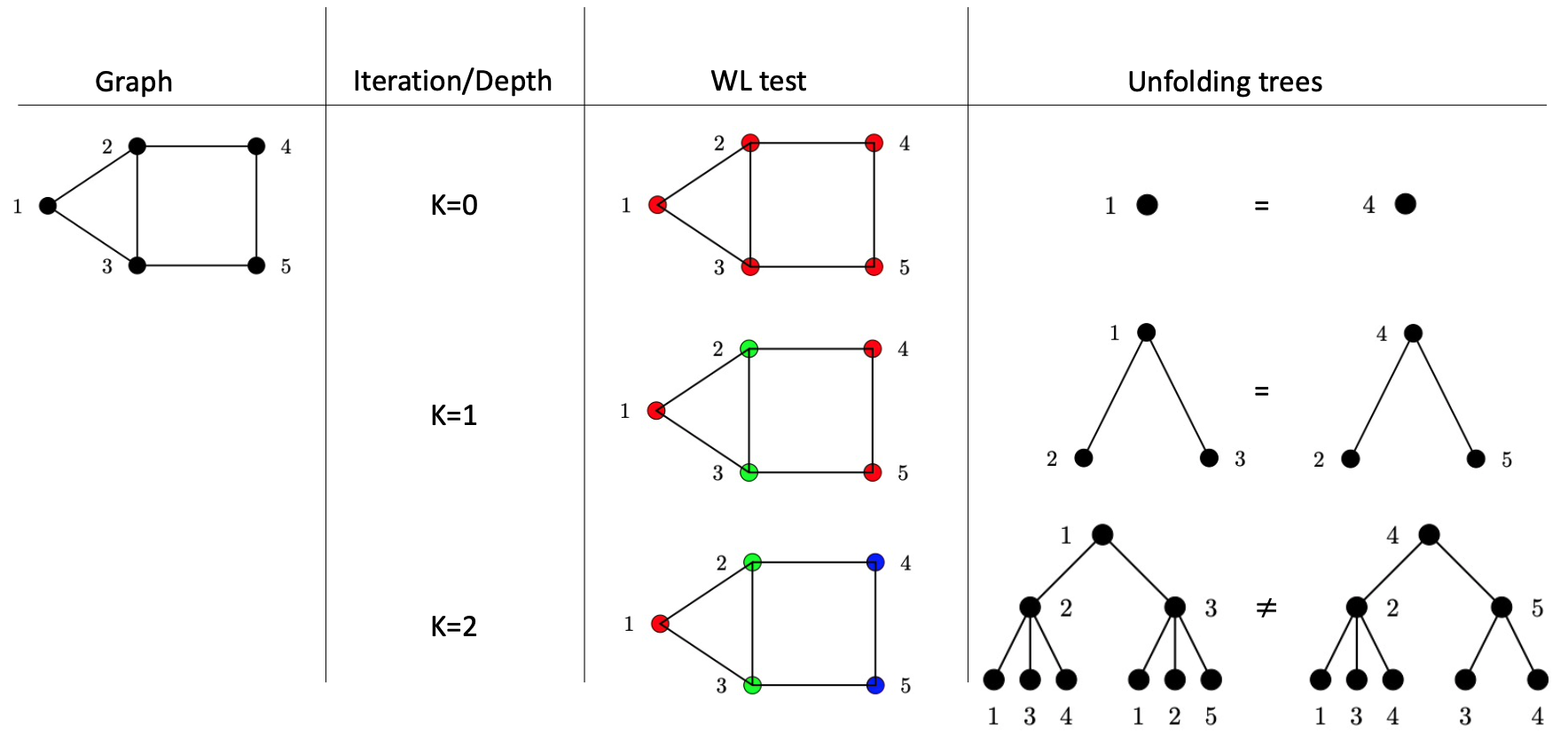}
\caption{A graphical representation of the relationship between the color refinement and the unfolding equivalence, applied on nodes 1 and 4 of the given graph.}
\label{fig:unfolding}
\end{figure*}

Indeed, the existence of a relationship between the equivalences appears to be a natural consequence of their definition. In fact, it is sometimes assumed in the literature (f.i., in~\cite{maron2018invariant}) that the two tools can be used interchangeably  but, as far as we know, there is no formal demonstration of their effective equivalence.  More precisely, in~\cite{krebs2015universal,Angluin80, dell2018lov},
it has been proved that the 1--WL test and unfolding trees produce the same profile on graph without
attributes. Therefore, Theorem~\ref{equiv_fin} is just an extension of those results to the case
of graph with attributes. On the other hand, Theorem~\ref{equiv_theo1}, focused on nodes, is completely novel.

Theorems~\ref{equiv_theo1} and   \ref{equiv_fin} are interesting since they formally confirm that the two equivalences are exactly interchangeable and can be used together to study GNNs. While the Weisfeiler--Lehman test has been often adopted to analyse the expressive power of GNNs in terms of their capability of recognizing different graphs, the unfolding equivalence and, more precisely, unfolding trees, can provide a  tool to understand the information that a GNN can use at each node to implement its function.

For example, it is well known that GNNs cannot distinguish regular graphs where nodes have the same features (see e.g. \cite{sato2020s}). Of course, in this case, a GNN is not able to distinguish any node, since all the unfolding trees are equal (see Figure \ref{fig:unfolding2}a). 
On the one hand, when a target node has different features with respect to the others,
also the unfolding trees incorporate such a difference and the nodes at different distances 
from this target node belong to different equivalence classes (see Figure \ref{fig:unfolding2}b).
On the other hand, if all the  labels are different, then each node belongs to a different
 class, since all unfolding trees are different (see Figure \ref{fig:unfolding2}c). 
 \begin{figure}[h!]
\centering
\subfigure[]{
\begin{tikzpicture}[scale=0.5]
\draw[fill=black] (0.5,0) circle (6pt);
\draw[fill=black] (2,2) circle (6pt);
\draw[fill=black] (4,2) circle (6pt);
\draw[fill=black] (5.5,0) circle (6pt);
\draw[fill=black] (2,-2) circle (6pt);
\draw[fill=black] (4,-2) circle (6pt);
\node at (0,0) {a};
\node at (1.5,2) {a};
\node at (4.5,2) {a};
\node at (1.5,-2) {a};
\node at (4.5,-2) {a};
\node at (6,0) {a};
\draw[thick] (0.5,0) -- (2,2) -- (4,2) -- (5.5,0) -- (4,-2) -- (2,-2) -- (0.5,0);
\end{tikzpicture}
}
\subfigure[]{
\begin{tikzpicture}[scale=0.5]
\draw[fill=red] (0.5,0) circle (6pt);
\draw[fill=yellow] (2,2) circle (6pt);
\draw[fill=green] (4,2) circle (6pt);
\draw[fill=black] (5.5,0) circle (6pt);
\draw[fill=yellow] (2,-2) circle (6pt);
\draw[fill=green] (4,-2) circle (6pt);
\node at (0,0) {b};
\node at (1.5,2) {a};
\node at (4.5,2) {a};
\node at (1.5,-2) {a};
\node at (4.5,-2) {a};
\node at (6,0) {a};
\draw[thick] (0.5,0) -- (2,2) -- (4,2) -- (5.5,0) -- (4,-2) -- (2,-2) -- (0.5,0);
\end{tikzpicture}
}
\subfigure[]{
\begin{tikzpicture}[scale=0.5]
\draw[fill=red] (0.5,0) circle (6pt);
\draw[fill=green] (2,2) circle (6pt);
\draw[fill=yellow] (4,2) circle (6pt);
\draw[fill=blue] (5.5,0) circle (6pt);
\draw[fill=orange] (2,-2) circle (6pt);
\draw[fill=black] (4,-2) circle (6pt);
\node at (0,0) {a};
\node at (1.5,2) {b};
\node at (4.5,2) {c};
\node at (1.5,-2) {d};
\node at (4.5,-2) {e};
\node at (6,0) {f};
\draw[thick] (0.5,0) -- (2,2) -- (4,2) -- (5.5,0) -- (4,-2) -- (2,-2) -- (0.5,0);
\end{tikzpicture}
}
\caption{(a) A regular graph where all nodes have the same features. All unfolding trees are equal.  (b) The equivalence classes when only one node has different features. (c) The equivalence classes when all nodes has different features.} 
\label{fig:unfolding2}
\end{figure}

We observe that, in principle, by adding random features to the node labels, we could make all the nodes distinguishable and improve the GNN expressive power. 
This fact was already mentioned for OGNNs \cite{Comp_GNN} and has been recently observed also for modern GNN models \cite{sato2021r}. Obviously, this is true only in theory, as the introduction of random features usually produces overfitting. However, some particular tasks exist where random features do not cause any overfitting, for example if these features are not related to the node content (see~\cite{Comp_GNN}, Section IV.A), while, in other cases, it is the particular model which is able to efficiently use random labels~\cite{sato2020s}.

A further important argument of our analysis regards how much deep must be unfolding trees, i.e., how many iterations of color refinement are needed, in order to make 
the equivalence stable. Actually, Theorems~\ref{equiv_theo1} and \ref{equiv_fin} suggest that the unfolding and Weisfeiler--Lehman equivalences
remain paired up to any depth/iteration $t$. Those equivalences naturally become finer and finer as the iterations proceed, i.e, $\backsim_{ue_{t-1}}\succ \backsim_{ue_t}$ and $\backsim_{WL_{t-1}}\succ \backsim_{WL_t}$, until $T$, when they  become stable and equal to the corresponding infinite equivalences, namely $\backsim_{ue_{T-1}}\equiv \backsim_{ue_T}\equiv \backsim_{ue}$ and $\backsim_{WL_{T-1}}\equiv \backsim_{WL_T}\equiv \backsim_{WL}$. As already mentioned in Section~\ref{notation}, according to the literature~\cite{kiefer2020iteration}, it is known that, for the WL--equivalence on graphs, $r-1$ is both
an upper and lower bound on $T$, where $r$ is the maximum number of nodes in the graphs. The following theorem, which takes inspiration from the results  in \cite{kiefer2020iteration} about covering trees, shows that, for equivalences \textit{on nodes}, the bounds are different  and we must wait up to $2r-1$ iterations,
i.e.,  trees of depth of $2r-1$,  until the equivalences become stable.

\begin{Thm}
The following statements hold for graphs with at most $r$ nodes.
\begin{enumerate}
\item
  Let $\mathbf{G}$ and $\mathbf{H}$ be connected graphs and $x,y$ be nodes of  $\mathbf{G}$ and $\mathbf{H}$, respectively. The infinite unfolding trees  $\mathbf{T}_x, \mathbf{T}_y$ are equal if and only if
    they are equal up to depth $2r-1$, i.e., $\mathbf{T}_x = \mathbf{T}_y$ iff $\;\mathbf{T}_x^{2r-1} = \mathbf{T}_y^{2r-1}$.
\item For any $r$, there exist two graphs $\mathbf{G}$ and $\mathbf{H}$ with nodes $x,y$, respectively,
    such that the infinite unfolding trees  $\mathbf{T}_x, \mathbf{T}_y$ are different, but they are equal up to
    depth $2r-16 \sqrt{r}$, i.e.,   $\mathbf{T}_x \neq \mathbf{T}_y$ and $\mathbf{T}_x^t = \mathbf{T}_y^t$ for $i \leq 2r-16 \sqrt{r}$. \hfill \qed
\end{enumerate}
\label{th:treeDepth}
\end{Thm}

In order to get an intuitive explanation of the reason why the bounds are different for graph and node equivalences, let us consider the case of two graphs $\mathbf{G}$ and $\mathbf{H}$ that are not equivalent,
i.e.,  $\mathbf{G}\not\equiv_{WL}\mathbf{H}$ holds. Moreover, let us assume that the parallel application of the
 refinement algorithm detects the difference in colors at iteration $\bar{T}$, namely
$\mathbf{G}\not\equiv_{WL_{\bar{T}}}\mathbf{H}$, for example because
 a new color is generated for graph $\mathbf{G}$ that is not present in $\mathbf{H}$. At this iteration, 
the WL algorithm is halted, since we detected at least a node $u$ in  $\mathbf{G}$ that is different from all
the nodes in $\mathbf{H}$. Conversely, if we continue the color refinement, the new color
of $u$ will generate other new colors, which are not present in $\mathbf{H}$,
 also for the neighbors of $u$. After at most $r$ iterations, the difference spreads throughout the graph,
 so that, finally, all the nodes in $\mathbf{G}$ are different from those in $\mathbf{H}$.
 This is intuitively correct, since all the nodes in $\mathbf{G}$ are connected to a node that does not exist
 in $\mathbf{H}$. Therefore, we can observe that, while the first difference
 between the nodes of the two graphs arises after $r-1$ iterations, the diffusion of such information
to all the nodes takes additional $r$ steps. Obviously, a similar conclusion can be derived also considering the unfolding equivalence and the depths of the unfolding trees.

An example that illustrates this situation is depicted in Figure~\ref{fig:counterexample}.
The two graphs in (a) and (b) 
 have been proposed in~\cite{krebs2015universal} and satisfy the lower bound of point 2 of 
 Theorem~\ref{th:treeDepth}. In the example, we assume that all the nodes have the same attributes, even if, for the sake of clarity, they are displayed with different symbols in terms of their "role" in the coloring scheme.
The graphs in (a) and (b) are constructed using  copies of the subgraph modules in (c), (d) and (e), which are merged in a sequence;
(a) and (b) are equal except at the top:  in (a), at the end of the sequence, there is a copy of  (d), while in (b) there is a copy of (e). 
The interesting case happens when the sequence is long enough so that $2r - 16\sqrt{r}>r$ holds.
In this  case,
we have the following situation: graphs (a) and (b) are distinguishable by the 1--WL test in less than $r$ steps; nevertheless, a number of steps $t >2r-16 \sqrt{r}>r$ is needed to distinguish the nodes $u$ and $v$.
Thus, intuitively, color refinement can recognize that (a) and (b)
are not isomorphic, but the detection of the difference occurs only 
when the information
about the asymmetry --- which is on one side of the sequence --- arrives to the other side of the sequence, where the different
modules have been placed. After that, the different modules have been detected and the information on their difference is propagated to the rest of the graphs in  a number of iterations proportional to the length of the sequences to
arrive back to nodes $u$ and $v$.

\begin{figure*}[ht]
\subfigure[]{
    \includegraphics[width=0.18\textwidth]{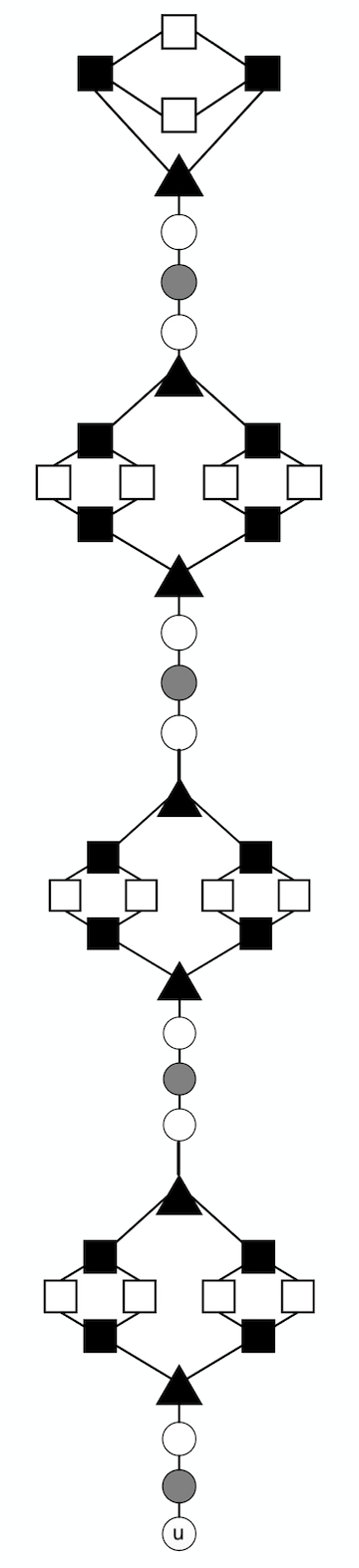}}
\subfigure[]{
    \includegraphics[width=0.18\textwidth]{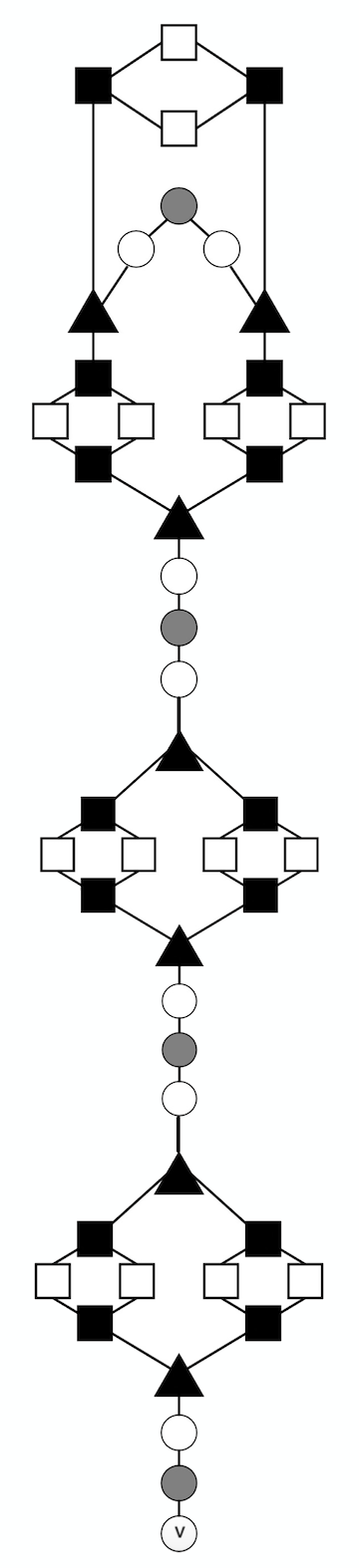}
}
\subfigure[]{
\includegraphics[width = 0.18\textwidth]{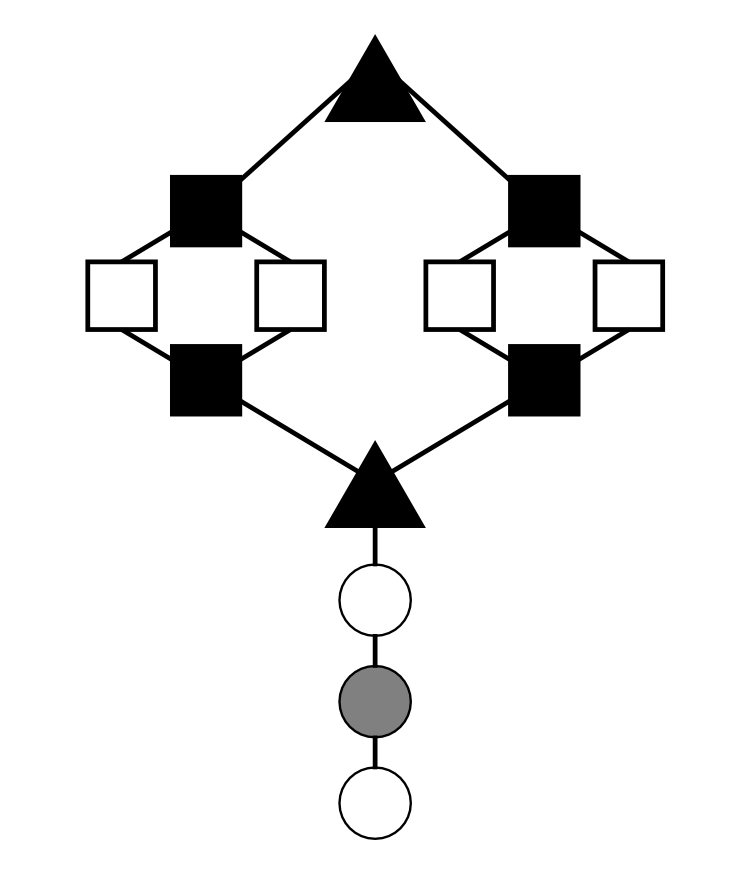}
}
\subfigure[]{
    \includegraphics[width=0.18\textwidth]{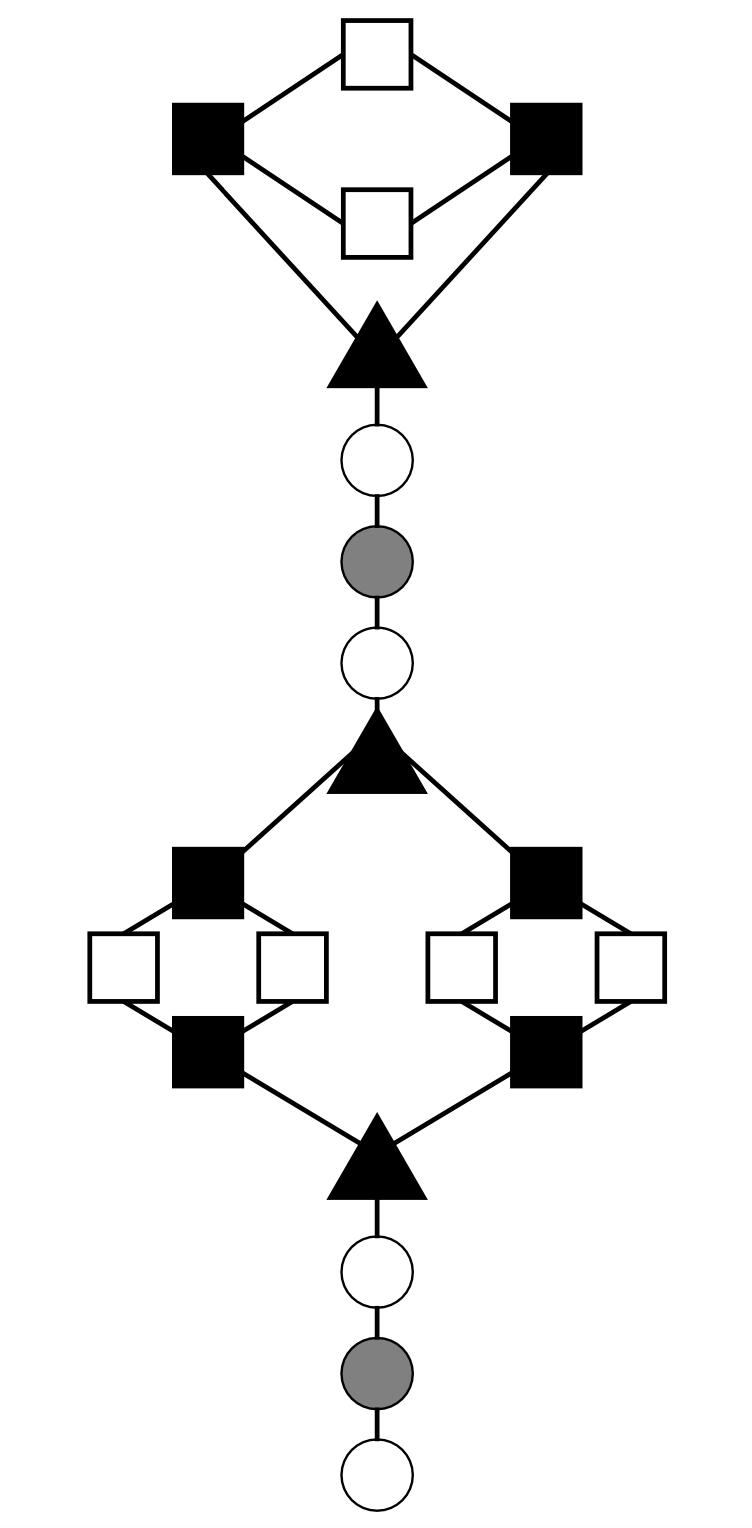}
}
\subfigure[]{
    \includegraphics[width=0.18\textwidth]{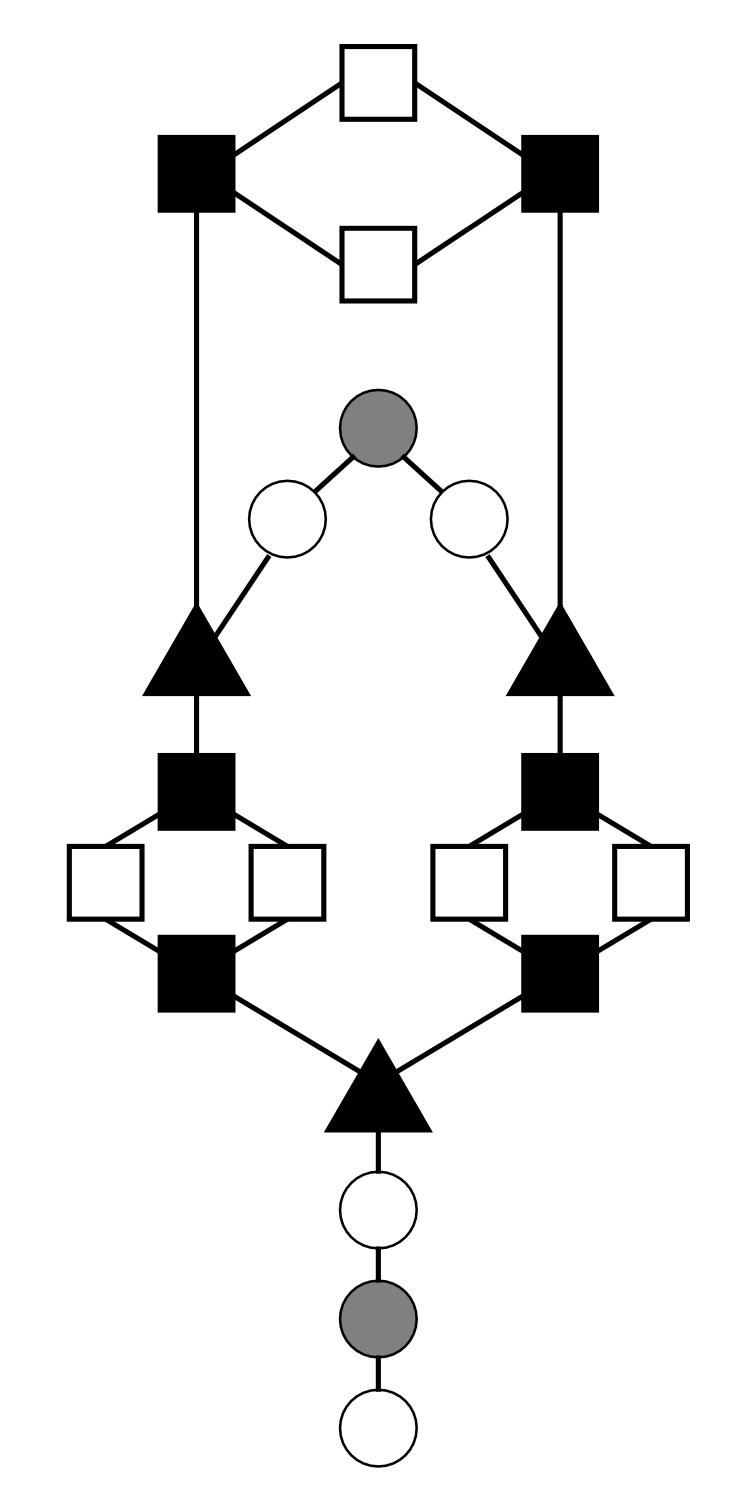}
}

\caption{In (a) and (b), two graphs $\mathbf{G}$, $\mathbf{H}$ are depicted that satisfy the lower bound of point 2 of of Theorem~\ref{th:treeDepth}. 
 We assume that all the nodes have the same attributes even if they are displayed with different symbols in terms of their "role" in the coloring scheme.
Graphs in (a) and (b) are constructed by aggregating in a sequence two copies of the same subgraph (c); then, module (d) is added at the top of graph (a), while module (e) is added at the top of graph (b).
It is worth noting that (a) and (b) \textit{do not} satisfy the relation $2r-16\sqrt{r}>r$; nevertheless, adding multiple times module (c) to the tail of both (a) and (b), we can find two graphs satisfying the requested relation. }
\label{fig:counterexample}
\end{figure*}

In order to formally link the concept of unfolding trees to the computational
capability of GNNs, let us now recall the definition of unfolding equivalence.

\begin{Def}\label{def:fun_unfold}
A function $f: \mathcal{D} \rightarrow \mathbb{R}^m$ is said to preserve the unfolding equivalence on $\mathcal{D}$ if
$v \sim u$ implies $f(\mathbf{G},v)=f(\mathbf{G}, u)$.
\hfill $\blacksquare$
\end{Def}

\vspace{.3cm}
The class of functions that preserve the unfolding equivalence on $\mathcal{D}$ will be denoted with $\mathcal{F}(\mathcal{D})$. A characterization of $\mathcal{F}(\mathcal{D})$ is given by the following result.

\begin{Thm}[Functions of unfolding trees] \label{f_unfold}
A function $f$ belongs to $\mathcal{F}(\mathcal{D})$ if and only if there exists a function $\kappa$, defined on trees, such that $f(\mathbf{G},v)= \kappa(\mathbf{T}^{2n-1}_v)$, for any node $v \in \mathcal{D}$.
\hfill \qed
\end{Thm}
A short, formal proof can be found in Appendix \ref{section_appendix}.

Theorem \ref{f_unfold} represents an improvement of the results reported in \cite{scarselli1998}; our contribution here is to show that, considering the unfolding tree down to the depth $2n-1$, we can provide the complete information on a graph to a function $f$ belonging to $\mathcal{F}(\mathcal{D})$. 

Note that Theorem \ref{f_unfold} suggests not only that the functions that compute the output on a node using unfolding trees preserve the unfolding equivalence, but also that the converse holds, namely all the functions that preserve the unfolding equivalence can be computed as functions of the unfolding trees. Since GNNs can implement only functions of the unfolding trees, we may expect that there is a tight relationship between what GNNs can do
and the class  $\mathcal{F}(\mathcal{D})$. Actually,  in \cite{Comp_GNN},
it has been shown that the OGNN model can approximate
in probability,  up to any degree of precision, any function in $\mathcal{F}(\mathcal{D})$ and a similar result will be derived for modern GNNs in this manuscript.

\subsection{Approximation capability}

The above discussion is about what GNNs cannot do, since we have proved that they are unable to distinguish nodes that originate equal unfolding trees.
Another obvious limit is that, at each node $v$, a GNN considers only the part of the graph that is reachable from $v$ and cannot implement any function depending on the information inaccessible from that node.
For this reason, for simplicity, we have decided to consider only connected graphs.
In this section, we pose our attention on two further questions that are related to each other, namely
 which functions can be approximated by GNNs and if there are any limitations other than that relating to the unfolding equivalence.

In order to address these issues, we consider the class of functions that preserve the unfolding equivalence (see Definition~\ref{def:fun_unfold}). The following theorem proves that GNNs can approximate in probability, up to any precision, any function of this
class, which means that GNNs are a sort of universal approximators on graphs, modulo the limitations due to the unfolding equivalence. 

\begin{Thm}[Approximation by GNNs] \label{main}
Let $\mathcal{D}$ be a  domain containing connected graphs with at most $r$ nodes.  For any measurable function $\tau \in \mathcal{F}(\mathcal{D})$ preserving 
the unfolding equivalence, any norm $\| \cdot \|$ on $\mathbb{R}$, and any probability measure $P$ 
on $\mathcal{D}$, there exists a GNN defined by the continuously differentiable functions $\text{COMBINE}^{(k)}$, $\text{AGGREGATE}^{(k)}$, $\forall k \leq r-1$, and by the function
$\text{READOUT}$, with feature dimension $m=1$ (i.e,  $h_v^k\in\mathbb{R}$), such that the function
$ \varphi$  (realized by the GNN) computed after $2r-1$ steps
satisfies the condition
\begin{equation*}
P( \| \tau(\mathbf{G},v)- \varphi( \mathbf{G},v) \| \leq \varepsilon) \geq 1- \lambda
\end{equation*}
for any reals $\epsilon, \lambda$, where $\epsilon >0$, $0 < \lambda < 1$.
\hfill \qed
\end{Thm}

Theorem~\ref{main} intuitively states that, given a function $\tau$, there exists a GNN that can approximate it. 
$\text{COMBINE}^{(k)}$ and $\text{AGGREGATE}^{(k)}$ can be any continuously differentiable function, while no assumptions are made on $\text{READOUT}$. This situation does not correspond to practical cases, where the GNN adopts
particular architectures and those functions are realized by neural networks or, more generally,
parametric models --- for example made of layers of sums, max, average, etc.
Therefore, it is of fundamental interest to clarify whether the theorem still holds
when the components $\text{COMBINE}^{(k)}$, $\text{AGGREGATE}^{(k)}$ and $\text{READOUT}$ are  parametric models. 

Let us now study the case when the employed 
components are sufficiently general to be able to approximate any function.
We call $\cal Q$ this class of networks, which corresponds to GNN models with universal components. In order to simplify our discussion, we introduce the transition function $f^{(k)}$  to indicate the stacking of the $\text{AGGREGATE}^{(k)}$ and $\text{COMBINE}^{(k)}$, i.e.,
\begin{align*}
f^{(k)}(\mathbf{h}_v^k,\{\mathbf{h}^{k-1}_u, \; u \in ne[v]\})& =  \\
=\text{COMBINE}^{(k)}\big(\mathbf{h}^{k-1}_v, & \\
\text{AGGREGATE}^{(k)}\{\mathbf{h}^{k-1}_u, \; u \in ne[v]\}\big)\,.&
\end{align*}
Then, we can formally define the class $\cal Q$.
\begin{Def} \label{def:universal}
A class $\cal{Q}$  of GNN models is said to have \textit{universal components} if, for any any $\epsilon>0$
and any continuous target functions $\overline{\text{COMBINE}}^{(k)}$, $\overline{\text{AGGREGATE}}^{(k)}$,  $\overline{\text{READOUT}}$, 
 there exists a GNN belonging to $\cal{Q}$, with functions $\text{COMBINE}_w^{(k)}$, $\text{AGGREGATE}_w^{(k)}$, $\text{READOUT}_w$ and
parameters $w$  such that
\begin{equation*}
\left\|{\bar f}^{(k)}(\mathbf{h},\{ \mathbf{h}_1,\ldots,\mathbf{h}_s\}) - f_w^{(k)}(\mathbf{h},\{ \mathbf{h}_1,\ldots,\mathbf{h}_s\}) 
\right\|_\infty\leq \epsilon
\end{equation*}
\begin{equation*}
\left\| \overline{\text{READOUT}}( \mathbf{q})-\text{READOUT}_w( \mathbf{q})\right\|_\infty\leq \epsilon
\end{equation*}
holds, for any input values  $\mathbf{h}$, $\mathbf{h}_1,\ldots,\mathbf{h}_s$, $\mathbf{q}$.
Here, the transition functions ${\bar f}^{(k)}$ and $f_w^{(k)}$ are defined using
 the target functions  $\overline{\text{COMBINE}}^{(k)}$, $\overline{\text{AGGREGATE}}^{(k)}$, and the GNN functions $\text{COMBINE}_w^{(k)}$, $\text{AGGREGATE}_w^{(k)}$, respectively, and $\|\cdot\|_\infty$ is the infinity norm.
\hfill $\blacksquare$
\end{Def}

\vspace{.3cm}

The following result shows that Theorem~\ref{main} still holds even for GNNs with universal components.

\begin{Thm}\emph{\sc Approximation by neural networks\/}\label{mainNN}\\
	Let us assume that the hypotheses of Theorem~\ref{main} are fulfilled and 
	$\cal{Q}$ is a class of  GNNs with universal components.
	Then, there exists a parameter set $w$ and some functions 
	$\text{COMBINE}^{(k)}_w$, $\text{AGGREGATE}^{(k)}_w$, $\text{READOUT}_w$,  implemented 	by neural networks in $\cal{Q}$, such that the thesis of	Theorem~\ref{main} holds. 
\label{nntheo}
\hfill \qed
\end{Thm}

The proof of Theorem~\ref{mainNN} is included in the Appendix. However, some related topics are discussed below, to better understand some properties of GNNs.
\begin{itemize}
\item
In the proof of Theorem~\ref{main}, we first define an encoding function $\triangledown$ (see the Appendix) that  maps trees to real numbers. The functions	$\text{COMBINE}^{(k)}$ and $\text{AGGREGATE}^{(k)}$ are designed so that, at each step, the node feature vector approximates a coding of the unfolding function $\mathbf{h}_v^k=\triangledown(\mathbf{T}_v^k)$. The function  $\text{READOUT}$ decodes the  unfolding and produces the desired outputs.
\item
In the proof of Theorem~\ref{mainNN}, it is shown that Theorem~\ref{main} still holds even when the transition  and $\text{READOUT}$ functions are approximated. Thus, we can use any parametric model to implement those functions.  We can expect that, also for the GNNs of Theorem~\ref{mainNN}, the transition function stores into the feature vector
an approximate coding of the unfolding tree, while  $\text{READOUT}$ decodes such a coding and gives the desired outputs. Obviously, in a practical case, a GNN can store only useful information,
required to produce the output, and not just all the informative content of the unfolding trees.
\end{itemize}
\noindent
The following remarks may further help to understand our results.
\begin{itemize}
\item
\noindent {\it GNNs with universal components}.
Intuitively, the universality condition means that the architectures used to implement
 $f_w^{(k)}$ and $\text{READOUT}_w$  must be sufficiently  general to  be able to approximate any possible target function. From the theory of standard neural networks, 
those architectures must have at least two layers (one hidden and one output layer) \cite{scarselli1998}. Such a conclusion is similar to the one reported in  \cite{xu2018powerful}, 
where a related result is described and where it is 
suggested that, in order to be able to implement the 1--WL test, the GNN must use a two layer transition function. Indeed, in this way, the GNN can implement
an injective encoding of the input graph into the node features.
Nonetheless, the proposed result is slightly different with respect to the one reported in \cite{xu2018powerful} as, in theory, the encoding may fail to be injective, provided
that the approximation remains sufficiently good in probability. However, the conclusion about the architecture still holds.

GNNs with transition functions  $f_w^{(k)}$ exploiting two layer architectures include Graph Isomorphism Networks (GINs)  \cite{xu2018powerful}, which were claimed to realize an injective encoding.  Similarly, also the OGNN model, for which a result similar to Theorem~\ref{main} was proved, adopts a two layer architecture for the transition function: in this case,  $\text{AGGREGATE}^{(k)}_w$  consists of a MultiLayer Perceptron (MLP)  with a hidden layer and $\text{COMBINE}^{(k)}_w$  was implemented by a sum. 
Similar results have been devised also in ~\cite{azizian2020expressive}, where a different version of the $\text{COMBINE}^{(k)}_w$ function has been modeled as a sum of MLPs.

\item {\it $\text{READOUT}$ universality.}
The condition on the universality of the   $\text{READOUT}$ function can be relaxed, provided that a higher dimension for the feature vector is used, namely $m>1$. $\text{READOUT}_w$ can indeed cooperate with the transition function in order to produce the output. In the limit case, the output can be completely prepared by the transition function and stored in some components of $\mathbf{h}_v^K$ so that  $\text{READOUT}_w$ is just a projection function.
\item {\it GNN architectures that are not universal approximators}. Most of GNN models,  e.g. Graph Convolutional Neural Networks, GraphSAGE and so on,  use a single layer architecture to implement the transition function. Thus, even if they do employ universal components, such as those specified by Definition 
\ref{def:universal}, they have a limited computational power with respect to two layer architectures and this is supported by theoretical results. In~\cite{xu2018powerful}, Lemma 7,  it is shown that, if the transition function is made up by a single layer with ReLU activation functions, the encoding function cannot be injective. A similar result was obtained for linear recursive neural networks~\footnote{Recursive neural networks~\cite{sperduti1997} are the ancestors of GNNs and assume that the input graph is  acyclic.} in~\cite{bianchini2001}. 
However, in general, it is not correct to assert that GNNs with single layer transition functions cannot be universal approximators for functions on graphs, as this property depends on 
the used GNN model and on other architectural/training details. For example,  a GNN model with a single layer transition component can use several iterations of Eq.~(\ref{Def_GNN}) to emulate a GNN with a deeper transition component. In the former model, the node features emulate  the 
transition network hidden layers and $\text{COMBINE}$  must contain a self--loop, namely
must have access to the previous features of each node.
\item {\it Feature dimension.} Surprisingly, Theorems~\ref{main} and~\ref{mainNN} suggest
that a feature vector of dimension $m=1$ is enough to establish the universal approximation capability of GNNs.
It is obvious, however, that the dimension of the feature vector plays an important role in determining
the complexity of the coding function for a given domain. We expect that the larger the dimension,
the smaller the complexity of the coding. This complexity, in turn, affects
the complexity of the transition function, the difficulty in learning such a function, 
the number of patterns required for training the GNN and so on.
\item {\it Number of steps.} Theorems~\ref{main} and~\ref{mainNN} suggest that $2r-1$ steps are enough to  approximate any function. Such a result is a consequence of Theorems~\ref{th:treeDepth}
and~\ref{f_unfold}.  Intuitively, this bound can be explained reusing the discussion on Theorem~\ref{th:treeDepth}. A GNN can employ up to $r-1$ iterations/layers to diffuse all the information
from one node to any other node with the message passing mechanism. 
After $r-1$ iterations, the information stored in a node provides
a sort of signature for that node, which may allow to distinguish some nodes from 
others. Yet, such a signature is not complete, since the first time
a node ``communicate'' with another has no information about itself. Adding $r$ iterations/layers allows nodes to communicate again and exchange their current signatures
to produce more accurate signatures. It is worth noting that this reasoning provides also an intuitive
explanation about why graph regression/classification tasks differ from node tasks. In graph tasks, the 
GNN  uses a $\text{READOUT}$  function that aggregates the features of all the nodes in the graph,
and possibly can do the work required by the second diffusion phase.
  In node tasks, $\text{READOUT}$   operates only on a single node, so that the second 
diffusion phase is mandatory.
\item  {\it The same $\text{COMBINE}$ and $\text{AGGREGATE}$ can be used for all the layers.}
Even if, for clarity, in our theoretical analysis, we focus on the GNN model that is the most used and exploits different functions in each layer $k$, 
our proofs do not exploit such a characteristic. Therefore, all the results hold also
for those GNNs --- sometimes called \textit{recursive} --- using the same $\text{COMBINE}$ and $\text{AGGREGATE}$ functions on each layer.
\end{itemize}

Note that, throughout the manuscript, we have used the idea that the unfolding tree represents the information available to a GNN to compute its output, and we have mentioned that a similar approach has been applied by other authors as well.
From a formal point of view,   Theorems~\ref{main} defines a method by which a GNN can actually encode an unfolding tree into the node features,
so that it has been proved that all the information collected into the unfolding trees can be used  by GNNs. 
However, also the reverse implication holds true, that is
a GNN cannot encode more information into features than that contained into the unfolding trees. 
Indeed, this is a
consequence of the fact that GNNs have no greater discriminatory capability than the 1--WL test (see~\cite{morris2019}, Theorem 1). Therefore, the unfolding trees totally collect the information used by a GNN.

Finally, the following corollary provides an alternative way to describe the approximation ability of GNNs as a function of their unfolding trees. 


\begin{Cor}
The class of functions implemented by a GNN with universal components is dense in probability in the $\mathcal{F}(\mathcal{D})$ class of functions that preserve the unfolding equivalence
in the domain $\mathcal{D}$ of connected graphs.
\qed
\end{Cor}

\section{Experimental Validation}\label{section_experiments}
In this Section, we support our theoretical findings with a set of experiments. For this purpose, we show that a GNN can approximate a function $F_{WL}:\mathcal{G}\rightarrow{\mathbb{N}}$ that models the 1--WL test. Indeed, the function $F_{WL}$ assigns to each graph a target label that represents the class of equivalence of the 1--WL. For simplicity, we only focus on the ability of the GNN to approximate this function, so that only the training performance is considered, i.e., we do not investigate its generalization capability over a test set. 
Since the 1--WL test provides the finest partition of graphs reachable by a GNN, the mentioned task experimentally establishes the expressive power of GNNs.

\paragraph{Dataset}
The graph datasets used for the experiments are derived from the QM9 molecules dataset \cite{ruddigkeit2012enumeration,ramakrishnan2014quantum}. Specifically, the subsets of molecules that compose our dataset are selected as follows:
\begin{itemize}
    \item Homogeneous features are assigned to all nodes of all graphs in QM9, as we are interested in evaluating the approximation ability of GNNs based only the graph topology;
    \item The 1--WL test is run all over the entire QM9 dataset for $k=4$ iterations, and for each graph, the target is the corresponding 1--WL output, represented as a natural number;
    \item We select the color classes containing more than $T$ graphs, where $T$ is a fixed threshold.
\end{itemize}
For training purposes, the targets are normalized between 0 and 1 and spaced uniformly in the range $[0,1]$.
Therefore, the distance between each class label is $d=\frac{1}{\text{num\_classes-1}}$. A graph $G$ with target $y_G$ will be said to be correctly classified if, given $\mathsf{out} = \text{GNN}(G)$, we have $|\mathsf{out}-y_G|< d/2 $.
\paragraph{Experimental setup}
The GNN used in the experiments is the Graph Isomorphism Network (GIN) \cite{xu2018powerful}.
A GIN  computes
\begin{equation}
    \mathbf{h}_v^{(t)} = \text{MLP}\big((1+\epsilon)\mathbf{h}_v^{(t-1)} + \sum \limits_{u \in ne_v} \mathbf{h}_u^{(t-1)}\big),
\end{equation}
where the attention parameter $\epsilon$ is either a trainable parameter or a fixed scalar. In our setting, we fix $\epsilon=0$. 
It has been proven that GINs  can implement 1-WL test and produce a  different  representation for each graph that can be distinguished by 1-WL test, ~\cite{xu2018powerful}.
Thus, GINs, with an appropriate READOUT,  can approximate any function on graphs preserving the unfolding equivalence. The MLP in a GIN layer has one hidden layer with $h_{\mathsf{gin}} $ neurons;
 the dimension of the GIN features is $h_{\mathsf{gin}} $ as well.
 The MLP in a GIN layer implements the hyperbolic tangent as activation function, and batch normalization. The GIN includes  $n_{\max}=k$ layers, so as $k$ is the number of  iterations performed by 1--WL to generate the targets. After the last GIN layer, the READOUT function is implemented performing a global aggregation, after which a linear layer  $W_{\mathsf{gin\_ out}}$ of size $h_{\mathsf{gin}} \times 1$ is added; eventually, a sigmoid activation function is applied.
The model is trained over $500$ epochs using the Adam optimizer with an initial learning rate $\lambda = 10^{-3}$. 
We carried out the experiments as follows.
\begin{itemize}
    \item In the first experimental setting, we evaluate the GNN performance for different values
    of the threshold $T$, which affects the cardinality of the  training set and its 1--WL color classes. The values of the threshold $T$ are taken in the integer interval $[30,45]$, the hidden layer of the MLP has dimension $h_{\mathsf{gin}} = 64$.
    \item In the second experimental setting, we evaluate the GNN expressive power varying both the number of neurons in the GIN MLP and the size of the hidden features, which, as specified above, are kept equal.
In these experiments,  the threshold is  fixed as $T=35$, the hidden layer sizes $h_{\mathsf{gin}}$ are taken from the list $[4,8,16,32,64]$.
\end{itemize}
Each experiment is statistically evaluated over $15$ runs. The overall training is  performed on an Intel(R) Core(TM) i7-9800X processor running at 3.80GHz, using 31GB of RAM and a GeForce GTX 1080 Ti GPU unit.
\begin{figure*}[tb!]
\centering
         \subfigure[]
        {\includegraphics[width=0.43\textwidth]
        {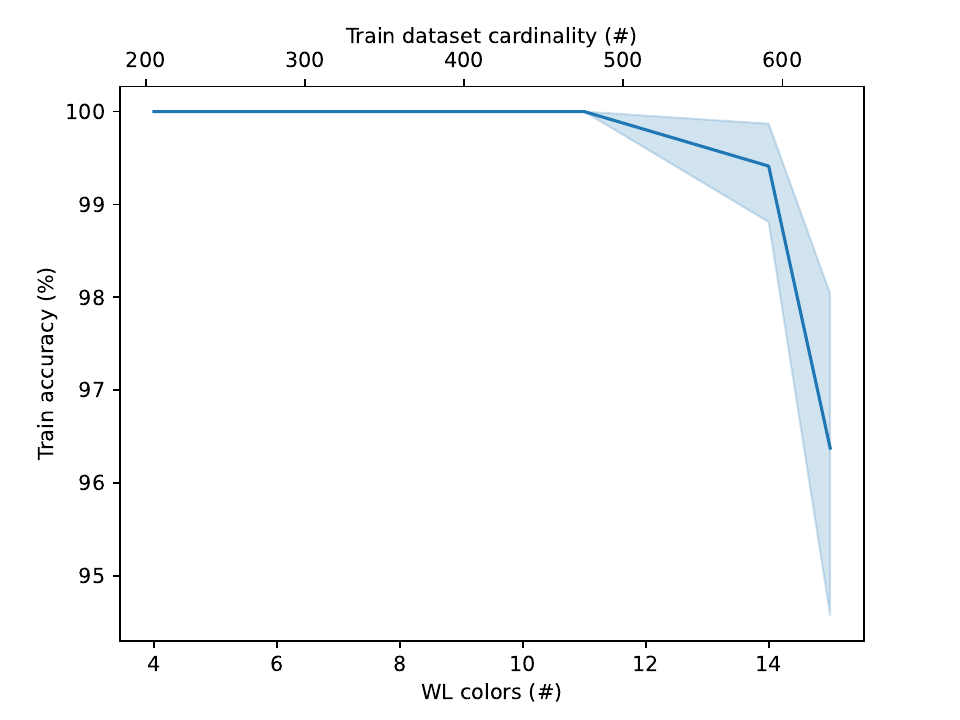} }
    \subfigure[]
        {\includegraphics[width=0.4\textwidth]{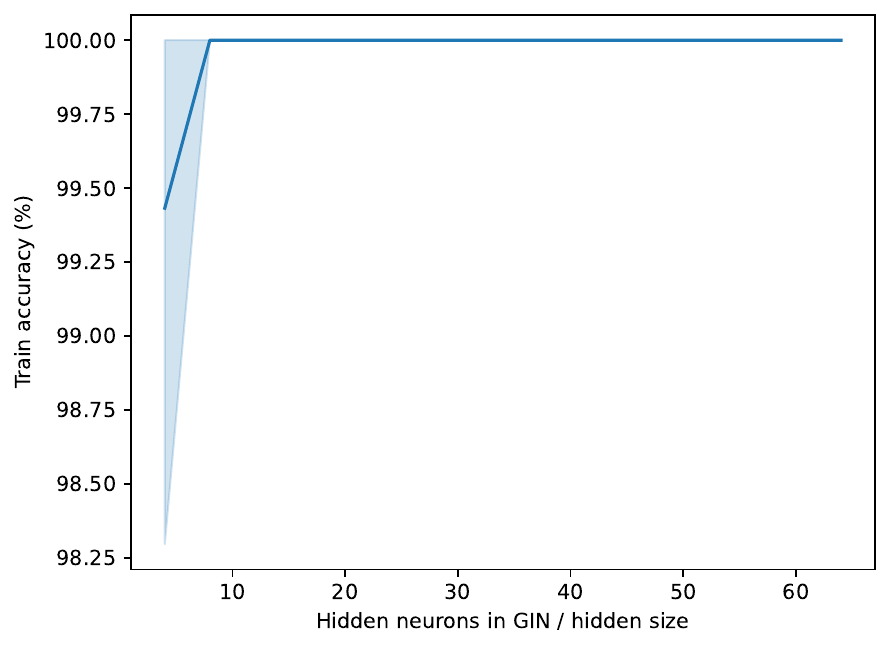}}
     \caption{Training accuracy on subsampled QM9 datasets, increasing number of WL colors (a), and increasing hidden layer size (b). The solid line represents the average over 15 runs, the shaded area represents the confidence interval.}
     \label{pic:experiments_acc}      
\end{figure*}
The code developed to run the experiments can be found at \url{https://github.com/AleDinve/static-gnn}.
\paragraph{Results}
Our experimental results are summarized in Figure \ref{pic:experiments_acc}.
Figure \ref{pic:experiments_acc} (a) shows the evolution of the training accuracy for different   numbers of WL colors;
Figure \ref{pic:experiments_acc} (b) displays the evolution of the training accuracy for  increasing  numbers of hidden neurons in the GIN MLP.\\
In both experiments the average training accuracy is never less than 96\%.
Moreover, in at least one of the 15 runs per value, 100\% training accuracy is reached. These results confirm the approximation power of GNNs equipped with a sufficiently general components.\\ 

\section{Conclusion}\label{ConcFut}
In this paper,  we have shown that GNNs can approximate, in probability, any function that preserves the unfolding equivalence (i.e., that passes the 1--WL test). Our proof improves on existing results both because it applies to node classification/regression tasks and because it is more general, since it holds for measurable functions.  Moreover, by using our theoretical framework, we have provided details on the GNN architectures that can reach a given approximation, including the number of iteration/layers,
the state dimension and the architecture of $\text{AGGREGATE}^{(k)}$, $\text{COMBINE}^{(k)}$ and  $\text{READOUT}$ networks. 

Future developments may include further extensions of our results beyond the traditional 1--WL domain and covering GNN models not considered by the framework used in this paper. For instance, it would be interesting to characterize the class of \textit{node--focused} functions learnable by a specific GNN model in terms of the isomorphism--wise test paradigm on which it has been built (see \cite{bodnar2021weisfeilera,bodnar2021weisfeilerb} for examples of GNNs built following isomorphism test mechanisms different from the 1--WL test).
Moreover, the proposed results are mainly focused on the expressive power of GNNs, but GNNs with the same expressive power may differ for other fundamental properties, e.g., the computational
and memory requirements and the generalization capability (that can be measured through well--established metrics, such as Rademacher complexity and VC--dimension, as pointed out in Section \ref{rel_work}, or evaluated in terms of neurocognitive task learning \cite{brugiapaglia2020generalizing,brugiapaglia2022invariance,d2023generalization}). Understanding how the architecture of  $\text{AGGREGATE}^{(k)}$, $\text{COMBINE}^{(k)}$ and  $\text{READOUT}$   impacts
on those properties is of fundamental importance for practical applications of GNNs.



\section*{Declarations}
\begin{itemize}
\item Funding: Giuseppe Alessio D'Inverno and Maria Lucia Sampoli are partially supported by INdAM GNCS group. Monica Bianchini and Maria Lucia Sampoli are partially supported by the PNRR Project "THE - Tuscany Health Ecosystem", CUP: B83C22003920001.
\item Conflict of interest/Competing interests: The authors have no relevant financial or non-financial interests to disclose.
\item Ethics approval: not applicable
\item Consent to participate: not applicable
\item Consent for publication: not applicable
\item Availability of data and materials: Synthetic data have been downloaded from the Pytorch Geometric repo available at \url{https://deepchemdata.s3-us-west-1.amazonaws.com/datasets/'
               'molnet_publish/qm9.zip}
\item Code availability: Code has been made available in the GitHub repo \url{https://github.com/AleDinve/static-gnn}.
\end{itemize}

\appendix
\section{Appendix}\label{section_appendix}

\section*{Proof of Theorems \ref{equiv_theo1} and \ref{equiv_fin}}

Since both unfolding equivalence and color equivalence have been described using a node-localized recursive definition, it is natural to investigate the possible connections between these two equivalence relations. Indeed, in the following, we show that they are equivalent on a domain of graphs with node features, i.e. that define the same relationship between nodes.

To prove Theorems \ref{equiv_theo1} and \ref{equiv_fin}, the following lemma is required.

\begin{Lem}\label{Lem_first}
\hspace{0.5 cm}\\
Let $\mathbf{G}=(\mathbf{V},\mathbf{E})$ be a graph and let $u,v \in \mathbf{V}$, with features $\ell_u,\, \ell_v$. Then,   $\forall t \in \mathbb{N}$
\begin{equation}
\mathbf{T}_u^t=\mathbf{T}_v^t \; \; \text{iff} \; \;  c_u^{(t)}=c_v^{(t)}
\label{eq_lem_first}
\end{equation}
where $c_u^{(t)}$ and $c_v^{(t)}$ represent the node coloring of $u$ and $v$ at time $t$, respectively.
\end{Lem}

\begin{proof}
The proof is carried out  by induction on $t$, which represents both the depth of the unfolding trees and the iteration step in the WL colouring.\\
For $t=0$, $\mathbf{T}_u^0= \text{Tree}(\ell_u)= \text{Tree}(\ell_v)= \mathbf{T}_v^0$ if and only if $\ell_u=\ell_v$ and $c_u^{(0)}=\text{HASH}_0(\ell_u)=\text{HASH}_0(\ell_v)=c_v^{(0)}$. Let us suppose that Eq. (\ref{eq_lem_first}) holds for $t-1$, and prove that it holds also for $t$.
\begin{itemize}
\item[($\rightarrow$)] Assuming that $\mathbf{T}_u^t=\mathbf{T}_v^t$, we have  
\begin{equation}\label{node}
\mathbf{T}_u^{t-1} = \mathbf{T}_v^{t-1}
\end{equation}
and
\begin{equation}\label{tree_eq}
\text{Tree}(\ell_u, \mathbf{T}_{ne[u]}^{t-1})=\text{Tree}(\ell_v, \mathbf{T}_{ne[v]}^{t-1})
\end{equation}
 By induction, Eq. (\ref{node})  is true if and only if 
\begin{equation}\label{col_node}
c_u^{(t-1)}=c_v^{(t-1)}
\end{equation}
Eq. (\ref{tree_eq}) implies that
$\ell_u=\ell_v$
and
$\mathbf{T}_{ne[u]}^{t-1}=\mathbf{T}_{ne[v]}^{t-1}$,
which means that an ordering on $ne[u]$ and $ne[v]$ exists s.t.
\begin{equation} \label{neig}
T_{ne_i(u)}^{t-1}=T_{ne_i(v)}^{t-1} \; \forall\  i=1, \dots, |ne[u]|
\end{equation}
Hence, Eq. (\ref{neig}) holds if and only if an ordering on $ne[u]$ and $ne[v]$ exists s.t.
\begin{equation*}
c_{ne(u)_i}^{t-1}=c_{ne(v)_i}^{t-1}\; \forall i=1, \dots, |ne[u]|
\end{equation*}
that is
\begin{equation}\label{col_neig}
\{ c_m^{(t-1)}| m \in ne[u] \}=\{ c_n^{(t-1)}| n \in ne[v] \}
\end{equation}
Putting together Eqs. (\ref{col_node}) and (\ref{col_neig}), we obtain:
\begin{equation*}
\text{HASH}(c_u^{(t-1)},\{ c_m^{(t-1)}| m \in ne[u] \} )=
\end{equation*}
\begin{equation*}
\text{HASH}(c_v^{(t-1)},\{ c_n^{(t-1)}| n \in ne[v] \} ).
\end{equation*}
which implies that
$c_u^{(t)}=c_v^{(t)}$.
\item[($\leftarrow$)]The proof of the converse implication follows a similar reasoning,
but some different steps are required in order to reconstruct the unfolding equivalence from the 
equivalence based on the 1--WL test.

Let us assume that $c_u^{(t)}=c_v^{(t)}$; by definition,

\begin{equation} \label{rev_col}
\text{HASH}(c_u^{(t-1)},\{ c_m^{(t-1)}| m \in ne[u] \} )=
\end{equation}
\begin{equation*}
\text{HASH}(c_v^{(t-1)},\{ c_n^{(t-1)}| n \in ne[v] \} )
\end{equation*}

Being the HASH function bijective, Eq. (\ref{rev_col}) implies that:

\begin{equation}\label{rev_col_node}
c_u^{(t-1)}=c_v^{(t-1)}
\end{equation}

and

\begin{equation}\label{rev_col_neig}
\{ c_m^{(t-1)}| m \in ne[u] \}=\{ c_n^{(t-1)}| n \in ne[v] \}
\end{equation}

Eq. (\ref{rev_col_node}) is true if and only if, by induction, 

\begin{equation}\label{rev_node}
\mathbf{T}_u^{t-1}=\mathbf{T}_v^{t-1}
\end{equation}

which implies

\begin{equation}\label{lab_eq}
\ell_u=\ell_v
\end{equation}

Moreover, Eq. (\ref{rev_col_neig}) means that an ordering on $ne[u]$ and $ne[v]$ exists such that
\begin{equation}\label{rev_col_sing}
c_{ne(u)_i}^{t-1}=c_{ne(v)_i}^{t-1}\; \forall i=1, \dots, |ne[u]|
\end{equation}
Instead, by induction, Eq. (\ref{rev_col_sing}) holds if and only if an ordering on $ne[u]$ and $ne[v]$ exists so as
$T_{ne_i(u)}^{t-1}=T_{ne_i(v)}^{t-1} \; \forall\  i=1, \dots, |ne[u]|$, i.e.
\begin{equation}\label{rev_neig}
\mathbf{T}_{ne[u]}^{t-1}=\mathbf{T}_{ne[v]}^{t-1}
\end{equation}

Finally, putting together Eqs. (\ref{lab_eq}) and (\ref{rev_neig}), we obtain

\begin{equation*}
\text{Tree}(\ell_u, \mathbf{T}_{ne[u]}^{t-1})=\text{Tree}(\ell_v, \mathbf{T}_{ne[v]}^{t-1})
\end{equation*}
that means $\mathbf{T}_u^t=\mathbf{T}_v^t$.
\end{itemize}
\end{proof}

Theorem \ref{equiv_theo1} is therefore proven, as its statement just rephrases the statement of Lemma \ref{Lem_first} in terms of the equivalence notation. Theorem \ref{equiv_fin} is the natural extension of Theorem \ref{equiv_theo1} to graphs.

\section*{Proof of Theorem \ref{th:treeDepth}}

In order to prove the theorem,  we introduce the concept of \textit{universal covering}, first presented in \cite{krebs2015universal}, which allows us to derive useful properties on the unfolding trees
(see \cite{krebs2015universal} for more details).

Let $\mathbf{G}=(\mathbf{V},\mathbf{E})$.
 Given a graph $\mathbf{H}=(\mathbf{V}',\mathbf{E}')$ and a homomorphism $\alpha$ from $\mathbf{H}$ to $\mathbf{G}$, if:
 \begin{itemize}
     \item $\alpha$ is a bijection from $ne(v)$ onto $ne(\alpha(v))$
      \item $f_v(v) = f_v(\alpha(v))$
      \item $f_v(u) = f_v(\alpha(u))$ $\forall u \in ne(v)$
      \end{itemize}
for all $v \in \mathbf{V}'$, then $\alpha$ is called an \textit{attributed covering map} and $\mathbf{H}$ is called a \textit{covering graph}.
Given a connected graph $\mathbf{G}$ and a vertex $x \in \mathbf{V}$, let us define a graph $\mathbf{U}_x(\mathbf{G})$ as follows. The vertex set of $\mathbf{U}_x(\mathbf{G})$ consists of all non--backtracking walks in $\mathbf{G}$ starting at $x$, that is, of sequences $(x_0, x_1, \dots , x_k)$ such that $x_0 = x$, $x_i$ and $x_{i+1}$ are adjacent, and $x_{i+1} \neq x_{i-1}$. 
Two such walks are \textit{adjacent} in $\mathbf{U}_x(\mathbf{G})$ if one of them extends the other by one component, that is, one is $(x_0, \dots, x_k, x_{k+1}) $ and the other is $(x_0, \dots, x_k)$. $\mathbf{U}_x(\mathbf{G})$ is a tree and $\gamma_G$ defined as $\gamma_G(x_0, \dots, x_k, x_{k+1}) = x_k$ is a covering map from $\mathbf{U}_x(\mathbf{G})$ to $\mathbf{G}$.
We call $\mathbf{U}$ an \textit{attributed universal cover} of $\mathbf{G}$ if $\mathbf{U}$ covers any covering graph of $\mathbf{G}$. Therefore, $\mathbf{U}_x(\mathbf{G})$ is an attributed universal cover of $\mathbf{G}$.

\begin{Rem}
    Given that we are dealing with attributed graphs, we will drop the "attributed" adjective from now on, to make the notation lighter.
\end{Rem}

The next lemma, which is proved in ~\cite{krebs2015universal},  shows the bijective correspondence between universal coverings and colors up to a certain depth/iteration.

\begin{Lem}\label{equiv_universal}{\cite{krebs2015universal}}
    Let $\mathbf{U}$ and $\mathbf{W}$ be universal covers of graphs $\mathbf{G}$ and $\mathbf{H}$, respectively. Furthermore, let $\alpha$ be a covering map from $\mathbf{U}$ to $\mathbf{G}$ and $\beta$ be a covering map from $\mathbf{W}$ to $\mathbf{H}$. Let $x \in \mathbf{V}(\mathbf{U})$ and $y \in \mathbf{V}(\mathbf{W})$, and let $u = \alpha (x) $ and $v = \beta (y)$. Then, for any $t$, $\mathbf{U}_x^t \cong \mathbf{W}_y^t$ if and only if $c^{(t)} (u) = c^{(t)} (v)$.
\end{Lem}

\begin{Rem}
    We can always identify the node $x$ from a covering $\mathbf{W}_x$ of a graph $\mathbf{H}$ with its mapping $u$ via $\alpha$; i.e., $x=u$. 
    This allows us to restate the previous bijection as: $\mathbf{U}_u^t \cong \mathbf{W}_v^t$ if and only if $c^{(i)} (u) = c^{(i)} (v)$.
\end{Rem}
We will now bridge the concepts of universal coverings and unfolding trees, passing through the colour refinement algorithm. 

\begin{Lem}\label{lem:unf_cov_eq}
    Let $\mathbf{G}$ and $\mathbf{H}$ be connected graphs and $x,y$ be nodes of  $\mathbf{G}$ and $\mathbf{H}$, respectively. Then $\mathbf{T}_x^t \cong \mathbf{T}_y^t$ if and only if $\mathbf{U}_x^t \cong \mathbf{W}_y^t$ for all $i$.
\end{Lem}

\begin{proof}
    Coupling Lemma \ref{Lem_first} and \ref{equiv_universal}, we straightforwardly obtain the result.
\end{proof}

The established bijection leads us directly to the proof of Theorem \ref{th:treeDepth}.
\begin{proof}{-- Proof of Theorem \ref{th:treeDepth}}.
The proof is based on the   reasoning adopted for Lemma 2.4 and Theorem 3.2 
in \cite{krebs2015universal}. Actually, such a lemma and theorem are similar to 
points (1) and (2) of  Theorem~\ref{th:treeDepth} and differ only because the results
in \cite{krebs2015universal} are about universal covers, whereas our points are about unfolding trees.
However, Lemma \ref{lem:unf_cov_eq} shows that universal covers and unfolding trees produce the same isomorphism on nodes.
\end{proof}

\section*{Proof of Theorem \ref{f_unfold} }
\begin{proof}
    It follows directly from the combination of Theorem 1 in~\cite{Comp_GNN} and Theorem \ref{th:treeDepth}. 
\end{proof}

\section*{Proof of Theorem \protect{\ref{main}} (Approximation by GNNs) }

First, we  need a preliminary lemma, for the the proof of which we refer to \cite{Comp_GNN}.
Intuitively, this lemma suggests that a graph domain with continuous features can be partitioned into 
small subsets so that the features of the graphs are almost constant in each partition. Moreover, a finite number of partitions is sufficient, in probability, to cover a large part of the domain.

\begin{Lem}{(Lemma 1 in  \cite{Comp_GNN})} \label{hypercubes}
For any probability measure $P$ on $\mathcal{D}$, and any reals $\lambda$, $\delta$, where $0 < \lambda \leq 1$, $\delta \geq 0$, there exist a real $\Bar{b} >0$, which is independent of $\delta$, a set $\Bar{\mathcal{D}} \subseteq \mathcal{D}$, and a finite number of partitions $\Bar{\mathcal{D}_1}, \dots , \Bar{\mathcal{D}_l}$ of $\Bar{\mathcal{D}}$, where $\Bar{\mathcal{D}} = \mathcal{G}_i \times \{ v_i \}$, with $\mathcal{G}_i \subseteq \mathcal{G}$ and $v_i \in \mathcal{G}_i$, such that: 
\begin{enumerate}
    \item $P(\Bar{\mathcal{D}}) \geq 1- \lambda $ holds;
    \item for each $i$, all the graphs in $\mathcal{G}_i$ have the same structure, i.e., they differ only for the values of their labels;
    \item for each set $\Bar{\mathcal{D}_i}$, there exists a hypercube $\mathcal{H}_i \in \mathbb{R}^a$ such that $\ell_{\mathbf{G}} \in \mathcal{H}_i$ holds for any graph $\mathbf{G} \in \mathcal{G}_i$, where $\ell_{\mathbf{G}}$ denotes the vector obtained by stacking all the feature vectors of $\mathbf{G}$;
    \item for any two different sets $\mathcal{G}_i$, $\mathcal{G}_j$, $i \neq j$, their graphs have different structures or their hypercubes $\mathcal{H}_i$, $\mathcal{H}_j$ have a null intersection, i.e. $\mathcal{H}_i \bigcap \mathcal{H}_j = \emptyset$;
    \item for each $i$ and each pair of graphs $\mathbf{G}_1$, $\mathbf{G}_2 \in \mathcal{G}_i$, the inequality $\| \ell_{\mathbf{G}_1} - \ell_{\mathbf{G}_2} \|_{\infty} \leq \delta$ holds;
    \item for each graph $\mathbf{G} \in \Bar{\mathcal{D}}$, the inequality $\| \ell_{\mathbf{G}}\|_{\infty} \leq \Bar{b}$ holds.
\end{enumerate}
\end{Lem}

By adopting an argument similar to that proposed in \cite{Comp_GNN}, it is proved that Theorem \ref{main} is equivalent to the following Theorem \ref{reduct}, where the domain contains a finite number of graphs and the features are integers.

\begin{Thm} \label{reduct}
For any finite set of patterns $\{ ( \mathbf{G}_i , v_i) |\  \mathbf{G}_i \in \mathcal{G}, v_i \in \mathit{N}, 1 \leq i \leq n \}$, with $ r= \max\limits_{\mathbf{G}_i} |(\mathbf{G}_i)|$ and with graphs having integer features, 
for any function $\tau : \mathcal{D} \rightarrow \mathbb{R}^m$, which preserves the unfolding equivalence, and for any real $\varepsilon >0$, there exist continuously differentiable functions  $\text{AGGREGATE}^{(k)}$, $\text{COMBINE}^{(k)}$, $\forall k \leq r+1$, s.t.
\begin{align*}
\mathbf{h}^k_v & = \text{COMBINE}^{(k)}\big(\mathbf{h}^{k-1}_{v}, & \\
& \text{AGGREGATE}^{(k)}\ \{\mathbf{h}^{k-1}_{u}, \; u \in ne[v]\}\big)&
\end{align*}
and a function
$\text{READOUT}$, with feature dimension $m=1$, i.e, $\mathbf{h}_v^k\in\mathbb{R}$, so that the function
$ \varphi$  (realized by the GNN), computed after $r+1$ steps,
satisfies the condition
\begin{equation}\label{mainIntEq}
|\tau (\mathbf{G}_i, v_i) - \varphi(\mathbf{G}_i, v_i)| \leq \varepsilon
\end{equation}
for any $i$, $1 \leq i \leq n$.
\end{Thm}

The equivalence is formally proved by the following lemma.

\begin{Lem}\label{reducLemma}
Theorem \ref{main} holds if and only if Theorem \ref{reduct} holds.
\end{Lem}
\begin{proof} Although the proof is quite identical to that contained in \cite{Comp_GNN}, we report it
here with the new notation.

Theorem \ref{main} is more general than Theorem \ref{reduct}, which makes this implication straightforward. Suppose instead that Theorem \ref{reduct} holds and show that this implies Theorem \ref{main}.
Let us apply Lemma \ref{hypercubes} with values for $P$ and $\lambda$ equal to the corresponding values of Theorem \ref{main}, being $\delta$ any positive real number. It follows that there is a real $\Bar{b}$ and a subset $\Bar{\mathcal{D}}$ of $\mathcal{D}$ s.t. $P(\Bar{\mathcal{D}}) > 1-\lambda$.
Let $\mathcal{M}$ be the subset of $\mathcal{D}$ that contains only the graphs $\mathbf{G}$ satisfying $\| \ell_{\mathbf{G}} \|_{\infty} \leq \Bar{b}$. Note that, since $\Bar{b}$ is independent of $\delta$, then $\Bar{\mathcal{D}} \subset \mathcal{M}$ for any $\delta$.
Since $\tau$ is integrable, there exists a continuous function which approximates $\tau$, in probability, up to any degree of precision. Thus, without loss of generality, we can assume that $\tau$ is equi--continuous on $\mathcal{M}$. By definition of equi--continuity, a real $\Bar{\delta} > 0 $ exists such that

\begin{equation}\label{cont_1}
    | \tau(\mathbf{G}_1,v) - \tau(\mathbf{G}_2,v) | \leq \frac{\varepsilon}{2}
\end{equation}
holds for any node $v$ and for any pair of graphs $\mathbf{G}_1, \mathbf{G}_2$ having the same structure and satisfying $\| \ell_{\mathbf{G}_1} - \ell_{\mathbf{G}_2} \|_{\infty} \leq \Bar{\delta}$.

Let us apply Lemma \ref{hypercubes} again, where, now, the $\delta$ of the hypothesis is set to $\Bar{\delta}$, i.e. $\delta= \Bar{\delta}$. From then on, $\Bar{\mathcal{D}} = \mathcal{G}_i \times \{ v_i \}$, $1 \leq i \leq n$, represents the set obtained by the new application of Lemma \ref{hypercubes} and  $I_i^{\Bar{b}, \Bar{\eta}}$ , $1 \leq i \leq 2d$, denote the corresponding intervals defined in the proof of the same lemma.
Let $\theta: \mathbb{R} \rightarrow \mathbb{Z}$ be a function that encodes reals into integers as follows: for any $i$ and any $z \in I_{i}^{\Bar{b}, \Bar{\eta}}$, $\theta(z) = i$. Thus, $\theta$ assigns to all the values of an interval $I_{i}^{\Bar{b}, \Bar{\eta}}$ the index $i$ of the interval itself. Since the intervals do not overlap and are not contiguous, $\theta$ can be continuously extended to the entire $ \mathbb{R}$. Moreover, $\theta$ can be extended also to vectors, being $\theta(\mathbf{Z})$ the vector of integers obtained by encoding all the components of $\mathbf{Z}$. Finally, let $\Theta: \mathcal{G}\rightarrow \mathcal{G}$ represent the function that transforms each graph by replacing all the feature labels with their coding, i.e. $\mathbf{L}_{\Theta(\mathbf{G})} = \theta(\mathbf{L}_{\mathbf{G}})$. Let $\Bar{\mathbf{G}_1}, \dots, \Bar{\mathbf{G}_v} $ be graphs, each one extracted from a different set $\mathcal{G}_i$. Note that, according to points 3, 4, 5 of Lemma \ref{hypercubes}, $\Theta$ produces an encoding of the sets $\mathcal{G}_i$. More precisely, for any two graphs $\mathbf{G}_1$ and $\mathbf{G}_2$ of $\Bar{\mathcal{D}}$, we have $\Theta(\mathbf{G}_1) = \Theta(\mathbf{G}_2)$ if the graphs belong to the same set, i.e., $\mathbf{G}_1, \mathbf{G}_2 \in \mathcal{G}_i$, while $\Theta(\mathbf{G}_1) \neq \Theta(\mathbf{G}_2)$ otherwise. Thus, we can define a decoding function $\Gamma$ s.t. $\Gamma(\Theta(\Bar{\mathbf{G}_i}), v_i)=(\Bar{\mathbf{\mathbf{G}_i}}, v_i)$, $1 \leq i \leq n$. 

Consider, now, the problem of approximating $\tau \circ \Gamma$ on the set $(\Theta(\Bar{\mathbf{G}_1}),v_1), \dots , ( \Theta(\Bar{\mathbf{G}_n}),v_n)$. Theorem \ref{reduct} can be applied to such a set, because it contains a finite number of graphs with integer labels. Therefore, there exists a GNN that implements a function $\Bar{\varphi}$ s.t., for each $i$, 
\begin{equation}\label{cont_2}
|\tau(\Gamma(\Theta(\Bar{\mathbf{G}_i}), v_i))- \Bar{\varphi}(\Theta(\Bar{\mathbf{G}_i}), v_i)| \leq \frac{\varepsilon}{2}
\end{equation}
However, this means that there is also another GNN that produces the same result operating on the original
graphs $\mathbf{G}_i$, namely a GNN for which
\begin{equation}\label{cont_3}
\varphi(\mathbf{G}_i, v_i) = \Bar{\varphi}( \Theta( \Bar{\mathbf{G}_i}), v_i)\,
\end{equation}
holds. Actually, the graphs ${\mathbf{G}_i}$ and $\Bar{\mathbf{G}_i}$ are equal except that the former
 has the coding of the feature labels attached to the nodes, while the latter contains the whole feature labels.
Thus, the GNN that operates on $\Bar{\mathbf{G}_i}$ is that suggested by Theorem \ref{reduct}, except that
 $\overline{\text{AGGREGATE}}^{(0)}$ preliminary creates a coding of $\theta(\ell_v)$.

Putting together the above equality with Eqs.
(\ref{cont_1}) and (\ref{cont_2}), it immediately follows that, for any $(\mathbf{G},v) \in \bar{\mathcal{D}_i}$,

\begin{equation*}
    |\tau(\mathbf{G}, v) - \varphi (\mathbf{G},v)| = 
\end{equation*}

\begin{equation*}
    = |\tau(\mathbf{G},v) - \tau (\Bar{\mathbf{G}_i},v) + \tau (\Bar{\mathbf{G}_i},v) - \varphi (\mathbf{G},v)|
\end{equation*}

\begin{equation*}
\leq | \tau(\Bar{\mathbf{G}_i},v)- \varphi(\mathbf{G},v)| +\frac{\varepsilon}{2}
\end{equation*}

\begin{equation*}
    = |\tau(\Gamma(\Theta(\Bar{\mathbf{G}_i}),v))- \Bar{\varphi}(\Theta(\Bar{\mathbf{G}_i}),v)| + \frac{\varepsilon}{2} \leq \varepsilon
\end{equation*}
Thus, the GNN described by Eq. (\ref{cont_3}) satisfies $|\tau(\mathbf{G},v) - \varphi (\mathbf{G},v) | \leq \varepsilon$ in the restricted domain $\bar{\mathcal{D}}$.
Since $P(\bar{\mathcal{D}}) \geq 1- \lambda$, we have:
\begin{equation*}
P( \| \tau(\mathbf{G},v)- \varphi( \mathbf{G},v) \| \leq \varepsilon) \geq 1- \lambda
\end{equation*}
which proves the lemma.

\end{proof}

\noindent
Now, we can  proceed to prove Theorem \ref{reduct}.

\begin{proof}[Proof of Theorem \ref{reduct}]
For the sake of simplicity, the theorem will be proved assuming $n=1$, i.e. $\tau(\mathbf{G},v) \in \mathbb{R}$. However, the result can be easily extended to the general case when $\tau(\mathbf{G},v) \in \mathbb{R}^n$. Indeed, in this case, the GNN that satisfies the theorem can be defined by stacking $n$ GNNs, each one approximating a component of $\tau(\mathbf{G},v)$.

According to Theorem \ref{f_unfold}, there exists a function $\kappa$ s.t. $\tau(\mathbf{G},v) = \kappa (\mathbf{T}_v)$.  Therefore,  an unfolding tree of depth $2r-1$, where $r$ is the maximum number of nodes in the graph domain, is enough to store the graph information, so that $\kappa$ can be designed to satisfy
$\tau(\mathbf{G},v) = \kappa (\mathbf{T}_v)= \kappa (\mathbf{T}_v^{2r-1})$; moreover, according to Theorem \ref{th:treeDepth}, the depth of the truncated unfolding tree is enough to respect the unfolding equivalence over all the nodes of every graph in the domain. 
Consequently, the main idea of the proof consists in designing a GNN that is able to encode the unfolding tree into the node features, i.e., for each
node $v$, we want to have $\mathbf{h}_v = \triangledown (\mathbf{T}_v^{2r-1})$, where $\triangledown$ is an encoding function that maps trees into real numbers.  More precisely, the encodings are constructed recursively  by  $\text{AGGREGATE}^{(k)}$ and $\text{COMBINE}^{(k)}$
functions using the neighbourhood information. After $t$ steps, the node features contain the encoding of the unfolding
tree $\triangledown (\mathbf{T}_v^t)$ of depth $t$. Then, after a number of steps $\bar{t}$ larger than the number of nodes of the graph, the GNN, by the READOUT function, can
produce the desired output $\kappa (\mathbf{T}_v^{n})$.

Accordingly, the theorem can be proved provided that we can implement the above mentioned procedure, which means that there exist
appropriate functions $\triangledown$, $\text{AGGREGATE}^{(k)}$, $\text{COMBINE}^{(k)}$ and READOUT. The existence of the READOUT function is obvious, since, given that unfolding trees can be encoded in node features, READOUT has just to decode the representation and compute the target output. Then, let use  focus on the other functions. They will be defined in two steps. Initially, $\text{AGGREGATE}^{(k)}$, $\text{COMBINE}^{(k)}$, and $\text{READOUT}$ will be defined without taking into account that they have to be continuously differentiable. Later,
this farther constraint will be considered.

\vspace{5pt}
\noindent 
\emph{The coding function $\triangledown$}\\
Let  $\triangledown$ be a composition of any two  injective functions $\alpha$ and $\beta$, $\alpha \circ \beta$, with the properties described in the following. 
\begin{itemize}
\item $\alpha$ is an injective function from the domain of the unfolding trees $\mathcal{T}^{N}$, calculated 
on the nodes of the graph $\mathbf{G}_i$,  to the Cartesian product $\mathbb{N} \times \mathbb{N}^P \times \mathbb{Z}^{\ell_{\mathbf{V}}}= \mathbb{N}^{P+1} \times \mathbb{Z}^{\ell_{\mathbf{V}}} $,  where $N$ is the number of nodes of the graph and $P$ is the maximum number of nodes a tree could have.

Intuitively, in the Cartesian product, $\mathbb{N}$ represents the tree structure, $\mathbb{N}^P$ denotes the node numbering, while, for each node, an integer vector $\in \mathbb{Z}^{\ell_{\mathbf{V}}}$ is used to encode the node features. 
Note that $\alpha$ exists and is injective, since the maximum information contained in an unfolding tree
is given by the union of all its node features and all its structural information, which is exactly equal to the codomain size of $\alpha$.
\item $\beta$ is an injective function from $\mathbb{N}^{P+1} \times \mathbb{Z}^{\ell_{\mathbf{V}}}$ to $\mathbb{R}$, whose existence is guaranteed by  the cardinality theory, since the two sets have the same cardinality. 
\end{itemize}
Since $\alpha$ and $\beta$ are injective, also the existence and the injectiveness of
 $\triangledown$ are ensured. 
 
 

 \vspace{7pt}
 \noindent
 \emph{Functions $\text{AGGREGATE}^{(k)}$ and $\text{COMBINE}^{(k)}$ }\\
Functions  $\text{AGGREGATE}^{(k)}$ and $\text{COMBINE}^{(k)}$ must satisfy 
\begin{equation*}
 \triangledown(\mathbf{T}_v^t)= \mathbf{h}_v^t= 
\end{equation*} 
\begin{equation*}
\begin{array}{lc}
 \text{COMBINE}^{(k)} \big(\mathbf{h}_v^{t-1}, & \\
 \text{AGGREGATE}^{(k)}\{\mathbf{h}^{t-1}_{u}, \; u \in ne[v]\}\big) &
 \end{array}
\end{equation*}
\begin{equation*}
\begin{array}{lc}
= \text{COMBINE}^{(k)} \big(\triangledown(\mathbf{T}_v^{t-1}), & \\ \text{AGGREGATE}^{(k)}\{\triangledown(\mathbf{T}^{t-1}_{u}), \; u \in ne[v]\}\big) &
\end{array}
\end{equation*}
\noindent
$\forall k \leq N$, where $N$ is the number of nodes. In a simple solution, $\text{AGGREGATE}^{(k)}$ decodes the trees of the neighbour  $\mathbf{T}^{t-1}_{u}$ of $v$ and stores them into a data structure
to be accessed by $\text{COMBINE}^{(k)}$. For example, the trees can be collected into the coding of a new  tree, i.e., $
\text{AGGREGATE}^{(k)}({\triangledown}(\mathbf{T}^{t-1}_{u} ), {u \in ne[v]})= {\triangledown}(\bigcup_{u \in ne[v]} {\triangledown}^{-1}(\triangledown(\mathbf{T}^{t-1}_{u})))$, 
where  $\bigcup_{u \in ne[v]} $ denotes an operator that  constructs a tree, with a  root having void features, from a set of sub--trees (see Figure  \ref{fig:union}). Then, $\text{COMBINE}^{(k)}$ assigns the correct features to the root by extracting them from  $\mathbf{T}^{t-1}_{v} $, i.e.,
\begin{equation*}
\begin{array}{lc}
\text{COMBINE}^{(k)}({\triangledown}(\mathbf{T}^{t-1}_{v}),b)  = &  \\
 {\triangledown}( \text{ATTACH}({\triangledown}^{-1}({\triangledown}(\mathbf{T}^{t-1}_{v})) ,{\triangledown}^{-1}(b)))&
 \end{array}
\end{equation*}

where ATTACH is an operator that  construct a tree following the procedure depicted in Figure \ref{fig:union} and $b$ is the result of the $\text{AGGREGATE}^{(k)}$ function. 
\begin{figure}[ht]
 \includegraphics[width=.99\linewidth]{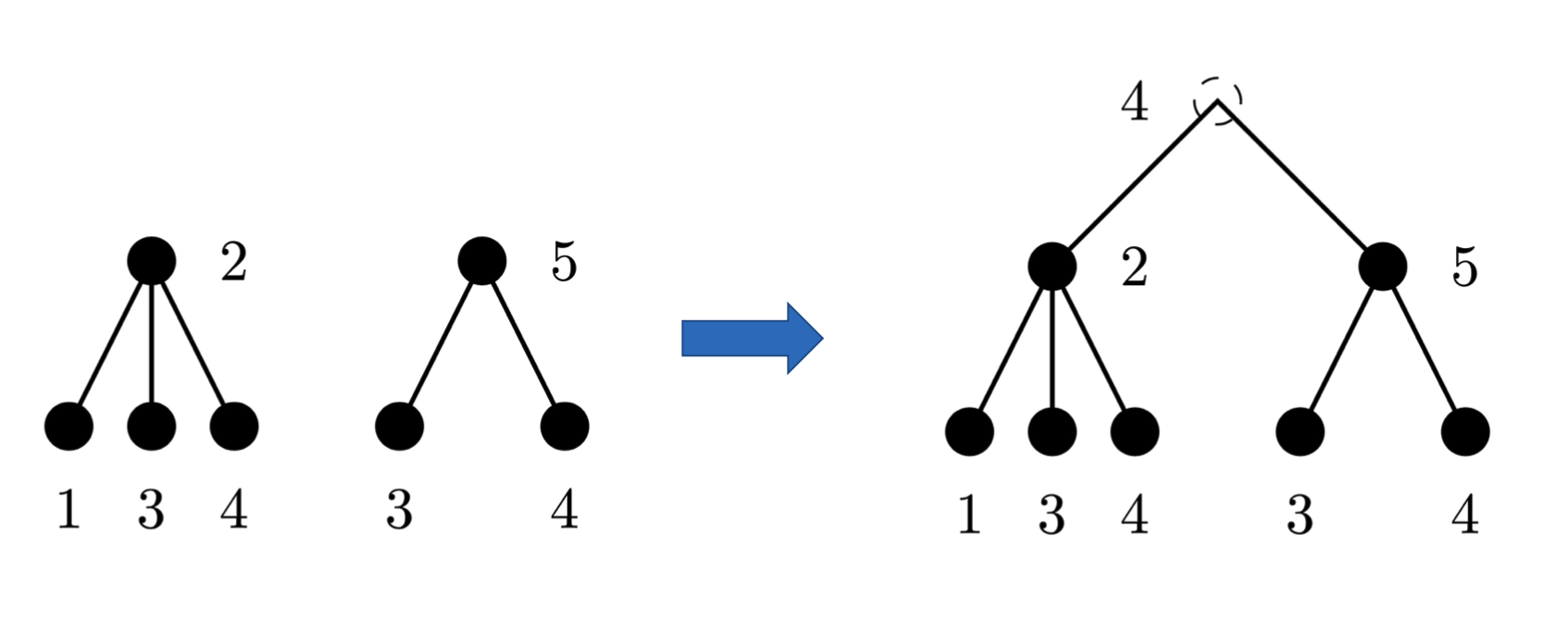}
\caption{The ATTACH operator on trees.}
\label{fig:union}
\end{figure}

Unfortunately, with this definition, $\text{AGGREGATE}^{(k)}$, $\text{COMBINE}^{(k)}$, and $\text{READOUT}$ may not be differentiable. Nevertheless, Eq.~(\ref{mainIntEq}) has to be satisfied only for a finite number of graphs,
namely $\mathbf{G}_i$. 
Thus, we can specify other functions $\overline{\text{AGGREGATE}}^{(k)}$, $\overline{\text{COMBINE}}^{(k)}$, and $\overline{\text{READOUT}}$, which  produce exactly the same computations when they are applied on the graphs 
$\mathbf{G}_i$, but that can be extended to the rest of their domain, so that they are continuously differentiable. Obviously, such an extension exists since those  functions are only constrained to interpolate a finite number of points\footnote{It is worth noting
that a similar extension can  also be applied to the coding function $\triangledown$ and to the
decoding function $\triangledown^{-1}$. In this case, the coding function is not injective on the whole domain, but only
on the graphs mentioned in the theorem.}.
\end{proof}

\section*{Proof sketch of Theorem \ref{nntheo}}
\begin{proof}

As in the proof of Theorem \ref{main}, without loss of generality, we will assume that the feature dimension is $m=1$. First of all, note that Theorem~\ref{main} ensures that we can find
$\overline{\text{COMBINE}}^{(k)}$, $\overline{\text{AGGREGATE}}^{(k)}$, $\forall k \leq N$, and $\overline{\text{READOUT}}$
so that, for the corresponding function  $\bar{\varphi}$ implemented by the GNN, 
\begin{equation}
\label{appMainEq}
P( \| \tau(\mathbf{G},v)- \bar{\varphi}( \mathbf{G},v) \| \leq \varepsilon/2) \geq 1- \lambda
\end{equation}
 holds. Let us consider the corresponding transition function $\bar{f}$, defined by 
\begin{equation*}
\begin{array}{lc}
\bar{f}^k(\mathbf{h}^{k-1}_{v},\{\mathbf{h}^{k-1}_{u}, u \in ne[v]\}) = &\\
=\overline{\text{COMBINE}}^{(k)}\big(\mathbf{h}^{k-1}_{v}, & \\
\overline{\text{AGGREGATE}}^{(k)}\{\mathbf{h}^{k-1}_{u}, u \in ne[v]\}\big) &
\end{array}
\end{equation*}
Since $\overline{\text{COMBINE}}^{(k)}$ and $\overline{\text{AGGREGATE}}^{(k)}$ 
are continuously differentiable, $\bar{f}^k$ is continuously differentiable. 
Considering that the theorem has to hold only in probability, we can also assume that the domain is bounded, so that
 $\bar{f}^k$ is bounded and has a bounded Jacobian. Let $B$ be a bound on the Jacobian/derivative of $\bar{f}^k$ for any $k$ and any input. The same argument can also be applied to the function
 $\overline{\text{READOUT}}$, which is continuously 
differentiable w.r.t. its input and can be assumed to have a bounded Jacobian/derivative. Let us assume that $B$
is also a bound for the Jacobian/derivative of $\overline{\text{READOUT}}$.
Moreover, let $\text{COMBINE}^{(k)}_w$ and $\text{AGGREGATE}^{(k)}_w$ be functions implemented by universal neural network that approximate $\overline{\text{COMBINE}}^{(k)}$, $\overline{\text{AGGREGATE}}^{(k)}$, $\forall k \leq r$, respectively,  and  such that
\begin{equation*} 
\begin{array}{lc}
f_w^k(\mathbf{h}^{k-1}_{v},\{\mathbf{h}^{k-1}_{u}, \; u \in ne[v]\}) = & \\ 
=\text{COMBINE}_w^{(k)}\big(\mathbf{h}^{k-1}_{v}, & \\
\text{AGGREGATE}^{(k)}_w\{\mathbf{h}^{k-1}_{u},  u \in ne[v]\}\big)\,. &
\end{array}
\end{equation*}
and let  us assume that
\begin{equation}\label{nnerror}
\|\bar{f}^k -f_w^k\|_\infty\leq \eta
\end{equation}
holds for every $k$ and a $\eta>0$. Let $\text{READOUT}_w$ be the function implemented by a universal neural network that approximates $\overline{\text{READOUT}}$,  so that
$$
\|\overline{\text{READOUT}}-\text{READOUT}\|_\infty\leq \eta
$$
In the following, it will be shown that, when $\eta$ is sufficiently small, the  GNN implemented by the approximating neural
networks is sufficiently  close to the GNN of Theorem \ref{main} so that the thesis is proved.

Let $\bar{F}^k$, $F_w^k$ be the global transition functions of the GNNs
that are obtained by stacking all the $\bar{f}^k$ and $f_w^k$ for all the nodes of the input graph. The node features are computed at each step by
$
\bar{H}^k={\bar F}^k(\bar{H}^{k-1}),\quad H^k=F_w^k(H^{k-1})
$, 
where $\bar{H}^k$,$H^k$ denote the stacking of all the node features of the graph
obtained by the two transition functions, respectively.
Then, 
\begin{equation}\label{h1bound}
\|\bar{H^1} -H^1\|_\infty=\|\bar{F}^1(H^0) -F_w^1(H^0)\|_\infty\leq \eta N
\end{equation}
where $N=|\mathbf{G}|$ is number of nodes in the input graph. Moreover,  
\begin{eqnarray*}
\lefteqn{\|\bar{H^2} -H^2\|_\infty=}&& \\
&=&\|\bar{F}^2(\bar{H}^1) -F_w^2(H)\|_\infty \\
&=&  \|\bar{F}^2(\bar{H}^1)-\bar{F}^2(H^1)+ \bar{F}^2(H^1)-F_w^2(H^1)\|_\infty\\
&\leq&
 \|\bar{F}^2(\bar{H}^1)-\bar{F}^2(H^1)\|_{\infty} +\| \bar{F}^2(H^1)-F_w^2(H^1)\|_\infty\\
&\leq& \eta N B +\eta N=\eta N(B+1)\,.
\end{eqnarray*}
Here, $\|\bar{F}^2(\bar{H}^1)-\bar{F}^2(H^1)\|_{\infty}\leq\eta N B $ holds because of  Eq. (\ref{h1bound}),
which bounds the difference between $\bar{H}^1$ and $H^1$, and due to the fact that the Jacobian/derivative of $\bar{F}^2$ is bounded by $B$.
Moreover, $\| \bar{F}^2(H^1)-F_w^2(H^1)\|_\infty\leq \eta N$ holds by Eq. (\ref{nnerror}). 

The above reasoning can then be applied recursively to prove that
$$
\|\bar{H}^k -H_w^k\|_\infty\leq \eta N\sum_{i=0}^{k-1} B^i
$$

Since the output of the GNN is computed using the encoding at step $N$,
we have 
\begin{eqnarray*}
\lefteqn{\|\overline{\varphi}(\mathbf{G},v)-\varphi_w(\mathbf{G},v)\|_\infty=}&&\\
&=&\|\overline{\text{READOUT}}(\overline{H}^{N})-\text{READOUT}_w(H^{N})\|_\infty\\
&\leq & \eta N+B (\eta N\sum_{i=0}^N B^i)
\end{eqnarray*}
Finally, since we can consider the maximum number of nodes $N$ as bounded\footnote{
For the sake of simplicity, we skip over a very formal proof of this claim.
Intuitively, note that the theorem has to be proved  and
Lemma \ref{hypercubes}  clarifies that any graph domain can be covered in high probability by
a finite number of structures, which obviously have a bounded number of nodes.
}, then
we can find  a GNN based on neural networks so that  $\eta$ is  small enough to achieve 
$$
\|\overline{\varphi}(\mathbf{G},v)-\varphi_w(\mathbf{G},v)\|_\infty\leq \epsilon/2
$$
which, together with Eq. (\ref{appMainEq}), produces the bound of Theorem~\ref{main}.
\end{proof}





\nocite*
\bibliography{sn-bibliography}

\end{document}